\pdfoutput=1
\documentclass[letterpaper,11pt]{article}
\usepackage{amsmath,amsthm,amssymb}
\usepackage{graphicx, hyperref, fullpage, color}
\usepackage{mathtools}
\usepackage[colorinlistoftodos,bordercolor=orange,backgroundcolor=orange!20,linecolor=orange,textsize=scriptsize]{todonotes}
\usepackage[numbers]{natbib}

\usepackage{algorithm}
\usepackage{algorithmic}

\usepackage{subfigure}

\usepackage{amsmath,amssymb, amsthm}
\usepackage{amsfonts}
\usepackage{layout}
\usepackage{amsmath}
\usepackage{amsthm}
\usepackage{amssymb}
\usepackage{url}
\usepackage{color}
\usepackage{graphicx}
\usepackage{verbatim}

\newcommand{\eqdef}{:=}

\newcommand{\R}{\mathbb{R}}
\newcommand{\Prob}{\mathbf{Prob}}
\newcommand{\E}{\mathbf{E}}

\newcommand{\bB}{\mathbf B}
\newcommand{\bM}{\mathbf M}

\newcommand{\bP}{\mathbf{P}}

\newcommand{\bX}{\mathbf{X}}

\DeclareMathOperator{\Diag}{Diag}       

\theoremstyle{plain}
\newtheorem{theorem}{Theorem}

\newtheorem{lemma}[theorem]{Lemma}

\theoremstyle{definition}

\newtheorem{definition}[theorem]{Definition}

\makeatletter
\newcommand*{\rom}[1]{\expandafter\@slowromancap\romannumeral #1@}
\makeatother

\title{Importance Sampling for Minibatches}

\author{Dominik Csiba and Peter Richt\'arik\thanks{This author would like to acknowledge support from the EPSRC Grant EP/K02325X/1, \textit{Accelerated Coordinate
Descent Methods for Big Data Optimization} and the EPSRC Fellowship EP/N005538/1, \textit{Randomized Algorithms for
Extreme Convex Optimization}.} \\ \\ \textit{School of Mathematics} \\ \textit{University of Edinburgh} \\ \textit{United Kingdom}}

\begin{document}

\maketitle

\begin{abstract} 
Minibatching is a very well studied and  highly popular technique  in supervised learning, used by practitioners due to its ability to accelerate training through better utilization of parallel processing power and reduction of stochastic variance. Another popular technique is importance sampling -- a strategy for preferential sampling of more important examples also capable of accelerating the training process. However, despite considerable effort by the community in these areas, and due to the inherent technical difficulty of the problem, there is no existing work combining the power of importance sampling with the strength of minibatching. In this paper we propose the first {\em importance sampling for minibatches} and give simple and rigorous complexity analysis of its performance. We illustrate on synthetic problems that for training data of certain properties, our sampling can lead to several orders of magnitude improvement in training time. We then test the new sampling on several popular datasets, and show that the improvement can reach an order of magnitude. 
\end{abstract} 

\section{Introduction}

Supervised learning is a widely adopted learning paradigm   with important applications such as regression, classification and prediction. The most popular approach to training supervised learning models is via empirical risk minimization (ERM). In ERM, the practitioner collects data composed of example-label pairs, and seeks to identify the best predictor by minimizing the empirical risk, i.e., the average risk associated with the predictor over the training data. 

With ever increasing demand for accuracy of the predictors, largely due to successful industrial applications, and with ever more sophisticated models that need to trained, such as deep neural networks \cite{DeepLearning,ConvNet}, or multiclass classification \cite{Multiclass}, increasing volumes of data are used in the training phase. This leads to huge  and hence extremely computationally intensive ERM problems. 

Batch algorithms---methods that need to look at all the data before taking  a single step to update the predictor---have long been known to be prohibitively impractical to use. Typical examples of batch methods are gradient descent and classical quasi-Newton methods. One of the most popular algorithms for overcoming the deluge-of-data issue is stochastic gradient descent (SGD), which can be traced back to a seminal work of  \citet{robbins1951}. In SGD, a single random example is selected in each iteration, and the predictor is updated using the information obtained by computing the gradient of the loss function associated with this example. This leads to a much more fine-grained iterative process, but at the same time introduces considerable stochastic noise, which eventually---typically after one or a few passes over the data----effectively halts the progress of the method, rendering it unable to push the training error (empirical risk) to the realm of small values.

\subsection{Strategies for dealing with stochastic noise}

Several approaches have been proposed to deal with the issue of  stochastic noise. The most important of these are  i) decreasing stepsizes, ii) minibatching, iii) importance sampling and  iv) variance reduction via ``shift'', listed here from historically first to the most modern. 

The first strategy, {\em decreasing stepsizes}, takes care of the noise issue by a gradual and direct scale-down process, which ensures that SGD converges to the ERM optimum \cite{tongSGD}. However, an unwelcome side effect of this is a  considerable slowdown of the iterative process \cite{BottouSGD}. For instance, the convergence rate is sublinear even if the function to be minimized is strongly convex.

The second strategy, {\em minibatching}, deals with the noise by utilizing a random set of examples in the estimate of the gradient, which effectively decreases the variance of the estimate \cite{shalev2011pegasos}. However, this has the unwelcome side-effect  of requiring more computation. On the other hand, if a parallel processing machine is available, the computation can be done concurrently, which ultimately leads to speedup. This strategy does not result in an improvement of the convergence rate (unless progressively larger minibatch sizes are used, at the cost of further computational burden \cite{friedlander2012hybrid}), but can lead to massive improvement of the leading constant, which ultimately means acceleration (almost linear speedup for sparse data) \cite{pegasos2}.

The third strategy, {\em importance sampling}, operates by a careful data-driven design of the probabilities of selecting examples in the iterative process, leading to a reduction of the variance of the stochastic gradient thus selected. Typically, the overhead associated with computing the sampling probabilities and with sampling  from the resulting distribution is negligible, and hence the net effect is speedup. In terms of theory, as in the case of minibatching, this strategy leads to the improvement of the leading constant in the complexity estimate, typically via replacing the maximum of certain data-dependent quantities by their average \cite{NSync, S2CD, IPROX-SDCA, Quartz, ImportanceSrebro, dfSDCA, AdaSDCA}.

Finally, and most recently, there has been a considerable amount of research activity due to the ground-breaking realization that one can gain the benefits of SGD (cheap iterations) without having to pay through the side effects mentioned above (e.g., halt in convergence due to decreasing stepsizes or increase of workload due to the use of minibatches). The result, in theory, is that for strongly convex losses (for example), one does not have to suffer sublinear convergence any more, but instead a fast linear rate ``kicks in''. In practice, these methods dramatically surpass all previous existing approaches.

The main algorithmic idea is to change the search direction itself, via a properly designed and cheaply maintainable {\em ``variance-reducing shift''} (control variate). Methods in this category are of two types: those operating in the primal space (i.e., directly on ERM) and those operating in a dual space (i.e., with the dual of the ERM problem).  Methods of the primal variety include SAG \cite{SAG}, SVRG \cite{SVRG}, S2GD \cite{S2GD},  proxSVRG \cite{proxSVRG}, SAGA \cite{SAGA}, mS2GD \cite{mS2GD} and MISO \cite{MISO}.  Methods of the dual variety work by updating randomly selected dual variables, which correspond to examples. These methods include SCD \cite{ShalevTewari11}, RCDM~\cite{Nesterov:2010RCDM, UCDC}, SDCA~\cite{SDCA}, Hydra~\cite{Hydra, Hydra2}, mSDCA~\cite{pegasos2}, APCG~\cite{APCG}, AsySPDC~\cite{WrightAsynchrous14}, RCD \cite{Necoara:rcdm-coupled},  APPROX \cite{APPROX}, SPDC~\cite{SPDC},  ProxSDCA~\cite{ProxSDCA}, ASDCA~\cite{ASDCA}, IProx-SDCA~\cite{IPROX-SDCA}, and QUARTZ~\cite{Quartz}.

\subsection{Combining strategies}

We wish to stress that the key strategies, mini-batching, importance sampling and variance-reducing shift, should be seen as orthogonal tricks, and as such they can be combined, achieving an amplification effect. For instance, the first primal variance-reduced method allowing for mini-batching was~\cite{mS2GD}; while dual-based methods in this category include~\cite{ASDCA,Quartz, dfSDCA}. Variance-reduced methods with importance sampling include~\cite{Nesterov:2010RCDM,UCDC,NSync, ALPHA} for general convex minimization problems, and~\cite{IPROX-SDCA, Quartz, ImportanceSrebro, dfSDCA} for ERM.

\section{Contributions}

Despite considerable effort of the machine learning and optimization research communities, no importance sampling for minibatches was previously proposed, nor analyzed.
The reason for this lies in the underlying theoretical and computational difficulties associated with the design and successful implementation of such a sampling. One needs to come up with a way to focus on a reasonable set of subsets  (minibatches) of the examples to be used in each iteration (issue: there are many subsets; which ones to choose?),  assign meaningful data-dependent non-uniform probabilities to them (issue: how?), and then be able to sample these subsets  according to the chosen distribution (issue: this could be computationally expensive). 

The tools that would enable one to consider these questions did not exist until recently. However, due to a recent line of work on analyzing variance-reduced methods  utilizing what is known as \emph{arbitrary sampling} \cite{NSync, Quartz, ALPHA, ESO, dfSDCA}, we are able to ask these questions and provide answers.
In this work we design a novel family of samplings---{\em bucket samplings}---and a particular member of this family---{\em importance sampling for minibatches}. We illustrate the power of this sampling in combination with the reduced-variance dfSDCA method for ERM. This method is a primal variant of SDCA, first analyzed by \citet{DualFreeSDCA}, and extended by \citet{dfSDCA} to the arbitrary sampling setting. However, our sampling can be combined with any stochastic method for ERM, such as SGD or S2GD, and extends beyond the realm of ERM, to  convex optimization problems in general. However, for simplicity, we do not discuss these extensions in this work.


We analyze the performance of the new sampling theoretically, and by inspecting the results we are able to comment on when can one expect to be able to benefit from it. We illustrate on synthetic datasets with varying distributions of example sizes that our approach  can lead to {\em dramatic speedups} when compared against standard (uniform) minibatching, of {\em one or more degrees of magnitude.} We then test our method on real datasets and confirm that the use of importance minibatching leads to up to an order of magnitude speedup. Based on our experiments and theory, we predict that for real data with particular shapes and distributions of example sizes,  importance sampling for minibatches will operate in a favourable regime, and can lead to speedup higher than one order of magnitude. 

\section{The Problem}

Let $\bX \in \R^{d \times n}$ be a data matrix in which features are represented in rows and examples in columns, and let $y \in \R^n$ be a vector of  labels corresponding to the examples. Our goal is to find a linear predictor $w\in \R^d$ such that $x_i^\top w \sim y_i$, where the pair $x_i,y_i \in \R^d \times \R$ is sampled from the underlying distribution over data-label pairs. In the L2-regularized Empirical Risk Minimization problem, we find $w$ by solving the  optimization problem
\begin{equation} \label{eq:problem}
\min_{w \in \R^d} \left[ P(w) := \frac{1}{n} \sum_{i=1}^n \phi_i(\bX_{:i}^\top w) + \frac{\lambda}{2} \|w\|_2^2 \right],
\end{equation}
where  $\phi_i:\R \rightarrow \R$ is a loss function associated with example-label pair $(\bX_{:i}, y_i)$, and $\lambda >0$. For instance, the square loss function is given by $\phi_i(t) = 0.5(t-y_i)^2$. Our results are not limited to L2-regularized problems though: an arbitrary strongly convex regularizer can be used instead \cite{Quartz}. We shall assume throughout that the loss functions are convex and $1/\gamma$-smooth, where $\gamma>0$. The latter means that for all $ x,y \in \R$ and all $i\in [n]:=\{1,2,\dots,n\}$, we have \[|\phi_i'(x) - \phi_i'(y)| \leq \frac{1}{\gamma}|x - y|.\] This setup includes ridge and logistic regression, smoothed hinge loss, and many other problems as special cases \cite{SDCA}. Again, our sampling can be adapted to settings with non-smooth losses, such as the hinge loss. 

\section{The Algorithm} 

In this paper we illustrate the power of our new sampling  in tandem with   Algorithm~\ref{alg:dfSDCA} (dfSDCA) for solving \eqref{eq:problem}.

\begin{algorithm}
\caption{dfSDCA~\cite{dfSDCA}} \label{alg:dfSDCA}
\begin{algorithmic}
\STATE \textbf{Parameters:} Sampling $\hat{S}$, stepsize $\theta>0$
\STATE \textbf{Initialization:} Choose $\alpha^{(0)} \in \R^n$, \\set $w^{(0)} = \frac{1}{\lambda n}\sum_{i=1}^n \bX_{:i}\alpha_i^{(0)},~p_i = \mathbf{Prob}(i \in \hat{S})$
\FOR{$t \geq 1$}
\STATE Sample a fresh random set $S_t$ according to $\hat{S}$
\FOR{$i \in S_t$}
\STATE $\Delta_i = \phi_i'(\bX_{:i}^\top w^{(t-1)}) + \alpha_i^{(t-1)}$
\STATE $\alpha_i^{(t)} = \alpha_i^{(t-1)} - \theta p_i^{-1}\Delta_i$
\ENDFOR
\STATE $w^{(t)} = w^{(t-1)} - \sum_{i\in S_t}\theta (n \lambda p_i)^{-1} \Delta_i \bX_{:i} $ 
\ENDFOR
\end{algorithmic}
\end{algorithm}

The method has two parameters. A ``sampling'' $\hat{S}$, which is a random set-valued mapping \cite{PCDM} with values being subsets of $[n]$, the set of examples. No assumptions are made on the distribution of $\hat{S}$ apart from requiring that $p_i$  is positive for each $i$, which simply means that each example has to have a chance of being picked. The second parameter is a stepsize $\theta$, which should be as large as possible, but not larger than a certain theoretically allowable maximum depending on $P$ and $\hat{S}$, beyond which the method could diverge.

Algorithm~\ref{alg:dfSDCA}  maintains $n$ ``dual'' variables, $\alpha_1^{(t)},\dots,\alpha_n^{(t)} \in \R$, which act as variance-reduction shifts. This is most easily seen in the case when we assume that $S_t=\{i\}$ (no minibatching). Indeed, in that case we have \[w^{(t)} = w^{(t-1)} - \frac{\theta}{n\lambda p_i}(g_{i}^{(t-1)} + \bX_{:i}\alpha_i^{(t-1)}),\] where $g_i^{(t-1)} \eqdef \bX_{:i}\Delta_i$ is the stochastic gradient. If $\theta$ is set to a proper value, as we shall see next, then it turns out that for all $i\in [n]$, $\alpha_i$ is converging $\alpha_i^*\eqdef -\phi_i'(\bX_{:i}^\top w^*)$, where $w^*$ is the solution to \eqref{eq:problem}, which means that the shifted stochastic gradient converges to zero. This means that its variance is progressively vanishing, and hence no additional strategies, such as decreasing stepsizes or minibatching are necessary to reduce the variance and stabilize the process.  In general, dfSDCA in each step picks a random subset of the examples, denoted as $S_t$,  updates  variables $\alpha_i^{(t)}$  for $i\in S_t$, and then uses these  to update the predictor $w$.

\subsection{Complexity of dfSDCA}

In order to state the theoretical properties of the method, we define \[ E^{(t)} \eqdef \frac{\lambda}{2}\|w^{(t)} - w^*\|_2^2 + \frac{\gamma}{2n} \| \alpha^{(t)} - \alpha^* \|_2^2. \]
Most crucially to this paper, we assume the knowledge of parameters $v_1, \dots, v_n>0$ for which the following ESO\footnote{ESO = Expected Separable Overapproximation \cite{PCDM, ESO}.} inequality holds
\begin{equation} \label{eq:ESO}
 \E \left[ \left\| \sum_{i \in S_t} h_i \bX_{:i}  \right\|^2 \right] \leq \sum_{i=1}^n p_i v_i h_i^2
\end{equation}
holds for all $h \in \R^n$. Tight and easily computable formulas for such parameters can be found in \cite{ESO}.  For instance, whenever $\Prob(|S_t|\leq \tau)=1$, inequality \eqref{eq:ESO} holds with $v_i = \tau \|\bX_{:i}\|^2$. However, this is a conservative choice of the parameters. Convergence of dfSDCA is described in the next theorem.

\begin{theorem}[\cite{dfSDCA}] \label{thm:dfSDCA}
Assume that all loss functions $\{\phi_i\}$ are convex and $1/\gamma$ smooth. If we run Algorithm~\ref{alg:dfSDCA} with parameter $\theta$ satisfying the inequality
\begin{equation} \label{eq:dfSDCA_theta}
\theta \leq \min_i \frac{p_i n \lambda \gamma}{ v_i + n \lambda \gamma},
\end{equation}
where $\{v_i\}$ satisfy \eqref{eq:ESO}, 
then  the potential $E^{(t)}$ decays exponentially to zero as
\[\E \left[E^{(t)}\right]
 \leq e^{-\theta t} E^{(0)}. \]
Moreover, if we set $\theta$ equal to the upper bound in \eqref{eq:dfSDCA_theta} so that
\begin{equation} \label{eq:complexity term}
 \frac{1}{\theta} = \max_i \left(\frac{1}{p_i}  + \frac{ v_i}{p_i n \lambda \gamma }\right)
\end{equation}
then 
\[ t\geq   \frac{1}{\theta} \log\left(\frac{(L+\lambda) E^{(0)}}{\lambda \epsilon}\right) \qquad \Rightarrow \qquad \E[P(w^{(t)}) - P(w^*)] \leq \epsilon.\]
\end{theorem}

\section{Bucket Sampling}

We shall first explain the concept of ``standard'' importance sampling.

\subsection{Standard importance sampling}\label{subsec:SIS}

Assume that $\hat{S}$ always picks a single example only. In this case, \eqref{eq:ESO} holds for $v_i = \|\bX_{:i}\|^2$, independently of $p := (p_1, \dots, p_n)$ \cite{ESO}. This allows us to choose the sampling probabilities as $p_i \sim v_i + n\lambda\gamma$, which ensures that   \eqref{eq:complexity term} is minimized. This is {\em importance sampling.}
The number of iterations of dfSDCA is in this case proportional to 
\[\frac{1}{\theta^{(\text{imp})}} \eqdef n +  \frac{\sum_{i=1}^n v_i}{n \lambda\gamma}.\] If uniform probabilities are used, the average in the above formula gets replaced by the maximum: \[\frac{1}{\theta^{\text{(unif)}}} \eqdef n +  \frac{\max_i v_i}{ \lambda\gamma}.\] Hence, one should expect the following {\em speedup} when comparing the importance and uniform samplings:
\begin{equation}\label{eq:sigma} \sigma \eqdef \frac{\max_i \|\bX_{:i}\|^2}{\frac{1}{n}\sum_{i=1}^n \|\bX_{:i}\|^2}.\end{equation}
If $\sigma=10$ for instance, then dfSDCA with importance sampling is 10$\times$ faster than dfSDCA with uniform sampling.

%

\subsection{Uniform minibatch sampling}

In machine learning, the term ``minibatch'' is virtually synonymous with a special sampling, which we shall here refer to by the name $\tau$-nice sampling \cite{PCDM}. Sampling $\hat{S}$ is $\tau$-nice if it picks uniformly at random from the collection of all subsets of $[n]$ of cardinality $\tau$. Clearly, $p_i=\tau/n$ and, moreover, it was show by \citet{ESO} that \eqref{eq:ESO} holds with $\{v_i\}$ defined by 
\begin{equation} \label{def:taunice_vi}
 v_i^{(\tau\text{-nice})} = \sum_{j=1}^d \left( 1 + \frac{(|J_j| - 1)(\tau - 1)}{n-1}\right)\bX_{ji}^2,
\end{equation} where $J_j \eqdef \{ i \in [n] :\bX_{ji} \neq 0 \}.$ In the case of $\tau$-nice sampling we have the stepsize  and complexity given by \begin{equation} \label{def:theta_taunice}
\theta^{(\tau\text{-nice})} = \min_{i}\frac{\tau \lambda \gamma}{v_i^{(\tau\text{-nice})} + n \lambda\gamma},
\end{equation}  \begin{equation}\label{eq:complexity_taunice}
\frac{1}{\theta^{(\tau\text{-nice})}} = \frac{n}{\tau} + \frac{\max_i v_i^{(\tau\text{-nice})}}{\tau\lambda\gamma}.
\end{equation} 

Learning from the difference between the uniform and importance sampling of single example (Section~\ref{subsec:SIS}), one would ideally wish the importance minibatch sampling, which we are yet to define, to lead to complexity of the type \eqref{eq:complexity_taunice}, where the maximum is replaced by an average. 

\subsection{Bucket sampling: definition}

We now propose a family of samplings, which we  call \textit{bucket samplings}. Let $B_1,\dots, B_\tau$ be a partition of $[n]=\{1,2,\dots,n\}$ into $\tau$ nonempty sets (``buckets''). 

\begin{definition}[Bucket sampling] We say that $\hat{S}$ is a bucket sampling if for all $i \in [\tau]$, $|\hat{S}\cap B_i| = 1$ with probability 1.
\end{definition}

Informally, a bucket sampling picks one example from each of the $\tau$ buckets, forming a minibatch. Hence, $|\hat{S}|=\tau$ and 
$\sum_{i\in B_l} p_i =1$ for each $l=1,2\dots,\tau$, where, as before, $p_i \eqdef \textbf{Prob}(i \in \hat{S})$. Notice that given the partition, the vector $p=(p_1,\dots,p_n)$ {\em uniquely determines} a bucket sampling. Hence, we have a family of samplings indexed by a single $n$-dimensional vector. Let  ${\cal P}_B$ be the set of all vectors $p\in \R^n$ describing bucket samplings associated with partition $B= \{B_1,\dots,B_\tau\}$. Clearly,
\[ {\cal P}_B = \left\{p\in \R^n : \sum_{i\in B_l} p_i =1 \text{ for all } l \;\&\;  p_i\geq 0 \text{ for all } i\right\}.\]

\subsection{Optimal bucket sampling}

The optimal bucket sampling is that for which \eqref{eq:complexity term} is minimized, which leads to a complicated optimization problem:
\[ \min_{p\in {\cal P}_B}   \max_i \frac{1}{p_i}  + \frac{ v_i}{p_i n \lambda \gamma }\quad\text{ subject to } \{v_i\} \text{ satisfy } \eqref{eq:ESO}.\]
A particular difficulty here is the fact that the parameters $\{v_i\}$ depend on the vector $p$ in a complicated way.  In order to resolve this issue, we prove the following result.

\begin{theorem}\label{thm:ESO} Let $\hat{S}$ be a bucket sampling described by partition $B=\{B_1,\dots,B_{\tau}\}$ and vector $p$. Then the ESO inequality \eqref{eq:ESO} holds for parameters $\{v_i\}$ set to
\begin{equation}\label{eq:new_v}  v_i = \sum_{j=1}^d  \left(1 + \left(1-\tfrac{1}{\omega'_j}\right) \delta_{j}\right)  \mathbf{X}^2_{ji},\end{equation}
where    $J_j \eqdef \{i\in [n] \;:\; \mathbf{X}_{ji} \neq 0\}$, $\delta_{j} \eqdef \sum_{i\in J_j} p_i$ and  $\omega'_j \eqdef \left| \left\{ l \;:\; J_j\cap B_l \neq \emptyset \right\}\right|$.
\end{theorem}

Observe that $J_j$ is the set of examples which express feature $j$, and $\omega'_j$ is the number of buckets intersecting with $J_j$. Clearly, that $1 \leq \omega'_j \leq \tau$ (if $\omega'_j=0$, we simply discard this feature from our data as it is not needed). Note that the effect of the quantities $\{\omega'_j\}$ on the value of $v_i$ is small. Indeed, unless we are in the extreme situation when $\omega'_j=1$, which has the effect of neutralizing $\delta_{j}$, the quantity $1-1/\omega'_j$ is between $1-1/2$ and $1-1/\tau$. Hence, for simplicity, we could instead use the slightly more conservative parameters: \[ v_i = \sum_{j=1}^d  \left(1 + \left(1-\frac{1}{\tau}\right) \delta_{j}\right)  \mathbf{X}^2_{ji} \].

\subsection{Uniform bucket sampling} 

Assume all buckets are of the same size: $|B_l|=n/\tau$ for all $l$.
Further, assume that  $p_i = 1/|B_l| = \tau/n$ for all $i$. Then $\delta_{j} = \tau |J_j| / n$, and hence Theorem~\ref{thm:ESO} says that  
\begin{equation}\label{eq:v_unif_bucket} v_i^{(\text{unif})} = \sum_{j=1}^d  \left(1 + \left(1-\frac{1}{\omega'_j}\right) \frac{\tau |J_j|}{n} \right)  \mathbf{X}^2_{ji},\end{equation}
and in view of \eqref{eq:complexity term}, the complexity of dfSDCA with this sampling becomes
\begin{equation}\label{eq:unif} \frac{1}{\theta^{(\text{unif})}} = \frac{n}{\tau} +\frac{\max_i  v_i^{(\text{unif})}}{\tau \lambda\gamma}.\end{equation}
Formula~\eqref{def:taunice_vi} is very similar  to the one for $\tau$-nice sampling \eqref{eq:v_unif_bucket}, despite the fact that the sets/minibatches generated by the uniform bucket sampling have a special structure with respect to the buckets.  Indeed, it is easily seen that the difference between between $ 1+ \tfrac{\tau |J_j|}{n} $ and $ 1+ \tfrac{(\tau-1) (|J_j|-1)}{(n-1)} $ is negligible. Moreover,  if either $\tau=1$ or $|J_j|=1$ for all $j$, then  $\omega_j' = 1$ for all $j$ and hence $v_i = \|\mathbf{X}_{:i}\|^2$. This is also what we get for the $\tau$-nice sampling.

 \section{Importance Minibatch Sampling}

In the light of Theorem~\ref{thm:ESO}, we can formulate  the problem of searching for the optimal bucket sampling as
\begin{equation}\label{eq:OPT_SOLVE}  \min_{p\in {\cal P}_B}   \max_i \frac{1}{p_i}  + \frac{ v_i}{p_i n \lambda \gamma }\quad \text{ subject to } \{v_i\} \text{ satisfy } \eqref{eq:new_v} .\end{equation}

Still, this is not an easy problem.  {\em Importance minibatch sampling} arises as an approximate solution of \eqref{eq:OPT_SOLVE}. Note that the uniform minibatch sampling is a feasible solution of the above problem, and hence we should be able to improve upon its performance.

\subsection{Approach 1: alternating optimization}

Given a probability distribution $p\in {\cal P}_B$, we can easily find  $v$ using Theorem~\ref{thm:ESO}. On the other hand, for any fixed $v$, we can minimize \eqref{eq:OPT_SOLVE}  over $p\in {\cal P}_B$ by choosing the probabilities in each group $B_l$ and for each $i\in B_l$ via \begin{equation}\label{eq:imp_prob}
 p_i = \frac{n\lambda\gamma + v_i }{\sum_{j \in B_l} n\lambda\gamma + v_j}.
\end{equation}
This leads to a natural alternating optimization strategy. Eventually, this strategy often (in experiments) converges to a pair $(p^*,v^*)$ for which \eqref{eq:imp_prob} holds. Therefore, the resulting complexity will be \begin{equation}\label{eq:complexity_imp}
\frac{1}{\theta^{(\tau\text{-imp})}} = \frac{n}{\tau} + \max_{l \in [\tau]} \frac{\frac{\tau}{n}\sum_{i \in B_l}v_i^*}{\tau\lambda\gamma}.
\end{equation}
We can compare this result against the complexity of $\tau$-nice in \eqref{eq:complexity_taunice}. We can observe that the terms are very similar, up to two differences. First, the importance minibatch sampling  has a maximum over group averages instead of a maximum over everything, which leads to speedup, other things equal.  On the other hand, $v^{(\tau\text{-nice})}$ and $v^*$ are different quantities.  The alternating optimization procedure for computation of $(v^*, p^*)$ is costly, as one iteration takes a pass over all data. Therefore, in the next subsection we propose a  closed form formula which, as we found  empirically, offers nearly optimal  convergence rate.

\subsection{Approach 2: practical formula}

 \label{sec:first_attempt}  For each group $B_l$, let us choose for all $i\in B_l$ the probabilities as follows:
\begin{equation}\label{eq:hoiho89898}\boxed{ p_i^* = \frac{ n\lambda \gamma + v_i^{(\text{unif})}}{\sum_{k \in B_l} n\lambda \gamma + v_k^{(\text{unif})} }}\end{equation}
where $v_i^{(\text{unif})}$ is given by \eqref{eq:v_unif_bucket}.  After doing some simplifications, the associated complexity result is
\[ \frac{1}{\theta^{(\tau\text{-imp})}} = 
\max_{l}  \left\{  \left(\frac{n}{\tau} + \frac{ \frac{\tau}{n}\sum_{i\in B_l} v_i^{(\text{unif})}}{\tau \lambda \gamma}  \right) \beta_l \right\},\]
where \[\beta_l \eqdef \max_{i\in B_l}  \frac{n\lambda \gamma + s_i}{n\lambda \gamma + v_i^{(\text{unif})}}\] and 
\[s_i \eqdef \sum_{j=1}^d  \left(1 +  \left(1-\frac{1}{\omega_j'}\right)\sum_{k\in J_j} p_k^*\right) \mathbf{X}^2_{ji}.\] We would ideally want to have  $\beta_l = 1$  for all $l$ (this is what we get for  importance sampling without minibatches).  If $\beta_l \approx 1$ for all $l$, then the complexity $1/\theta^{(\tau\text{-imp})}$  is an improvement on the complexity of the uniform minibatch sampling since the maximum of group averages  is always better than the maximum of all elements $v_i^{(\text{uni})}$:
\[  \frac{n}{\tau} + \frac{\max_l \left( \frac{\tau}{n}\sum_{i\in B_l} v_i^{\text{(unif)}}\right)}{\tau \lambda \gamma}  \leq \frac{n}{\tau} + \frac{\max_i  v_i^{\text{(unif)}}}{\tau \lambda \gamma}.\]
Indeed, the difference can be very large. 


\section{Experiments}

We now comment on the results of our numerical experiments, with both synthetic and real  datasets. We plot the optimality gap $P(w^{(t)}) - P(w^*)$ (vertical axis) against the computational effort (horizontal axis). We measure computational effort by the number of effective passes through the data  divided by $\tau$. We divide by  $\tau$ as a normalization factor; since we shall compare methods with a range of values of $\tau$. This is reasonable as it simply indicates that the $\tau$  updates are performed in parallel. Hence, what we plot is an implementation-independent model for time.

We compared two algorithms: \begin{itemize} 
\item[1)] \textbf{$\tau$-nice}: dfSDCA using the $\tau$-nice sampling with stepsizes given by \eqref{def:theta_taunice} and \eqref{def:taunice_vi}, 
\item[2)] \textbf{$\tau$-imp}: dfSDCA using $\tau$-importance sampling (i.e., importance minibatch sampling) defined in Subsection~\ref{sec:first_attempt}. \end{itemize} For each dataset we provide two plots. In the left figure we plot the convergence of $\tau$-nice for different values of $\tau$, and in the right figure we do the same for $\tau$-importance. The horizontal axis has the same range in both plots, so they are easily comparable. The values of $\tau$ we used to plot are $\tau \in \{ 1,2,4,8,16,32\}$. In all  experiments we used the  logistic loss: $\phi_i(z) = \log(1 + e^{-y_i z})$ and set the regularizer to $\lambda	 = \max_i \|\bX_{:i}\| /n$.  We will observe the theoretical and empirical ratio $\theta^{(\tau\text{-imp})}/\theta^{(\tau\text{-nice})}$. The theoretical ratio is computed from the corresponding theory. The empirical ratio is the ratio between the horizontal axis values at the moments when the algorithms reached the precision $10^{-10}$.

\subsection{Artificial data}

We start with experiments using artificial data, where we can control the sparsity pattern of $\bX$ and the distribution of $\{\|\bX_{:i}\|^2\}$. We fix $n = 50,000$ and choose $d = 10,000$ and $d = 1,000$. For each feature we sampled a random sparsity coefficient $\omega'_i \in [0,1]$ to have the average sparsity  $\omega' \eqdef \frac{1}{d}\sum_{i}^d \omega'_i$ under control. We used two different regimes of sparsity: $\omega' =0.1$ (10\% nonzeros) and $\omega'=0.8$ (80\% nonzeros).  After deciding on the sparsity pattern, we rescaled the examples to match a specific distribution of norms $L_i=\|\bX_{:i}\|^2$; see Table~\ref{tab:dists}. The code column shows the corresponding code in Julia to create the vector of norms $L$. The distributions can be also observed as histograms in Figure~\ref{fig:hists}.

\setlength\tabcolsep{5pt}

\begin{table}[!ht]

\centering 
\vskip 0.3cm
\begin{tabular}{|l|l|l|}
\hline
label     & code & $\sigma$\\ \hline
extreme &  L = ones(n);L[1] = 1000 & 980.4   \\ \hline
chisq1	& L = rand(chisq(1),n) & 17.1 \\ \hline
chisq10	& L = rand(chisq(10),n) & 3.9\\ \hline
chisq100	& L = rand(chisq(100),n)& 1.7 \\ \hline
uniform & L = 2*rand(n) & 2.0 \\ \hline
\end{tabular}
\caption{Distributions of $\|\bX_{:i}\|^2$ used in artificial experiments.}
\label{tab:dists}
\end{table}

\begin{figure}[!ht]
\hfill
\subfigure[uniform]{\includegraphics[width = 0.32\columnwidth]{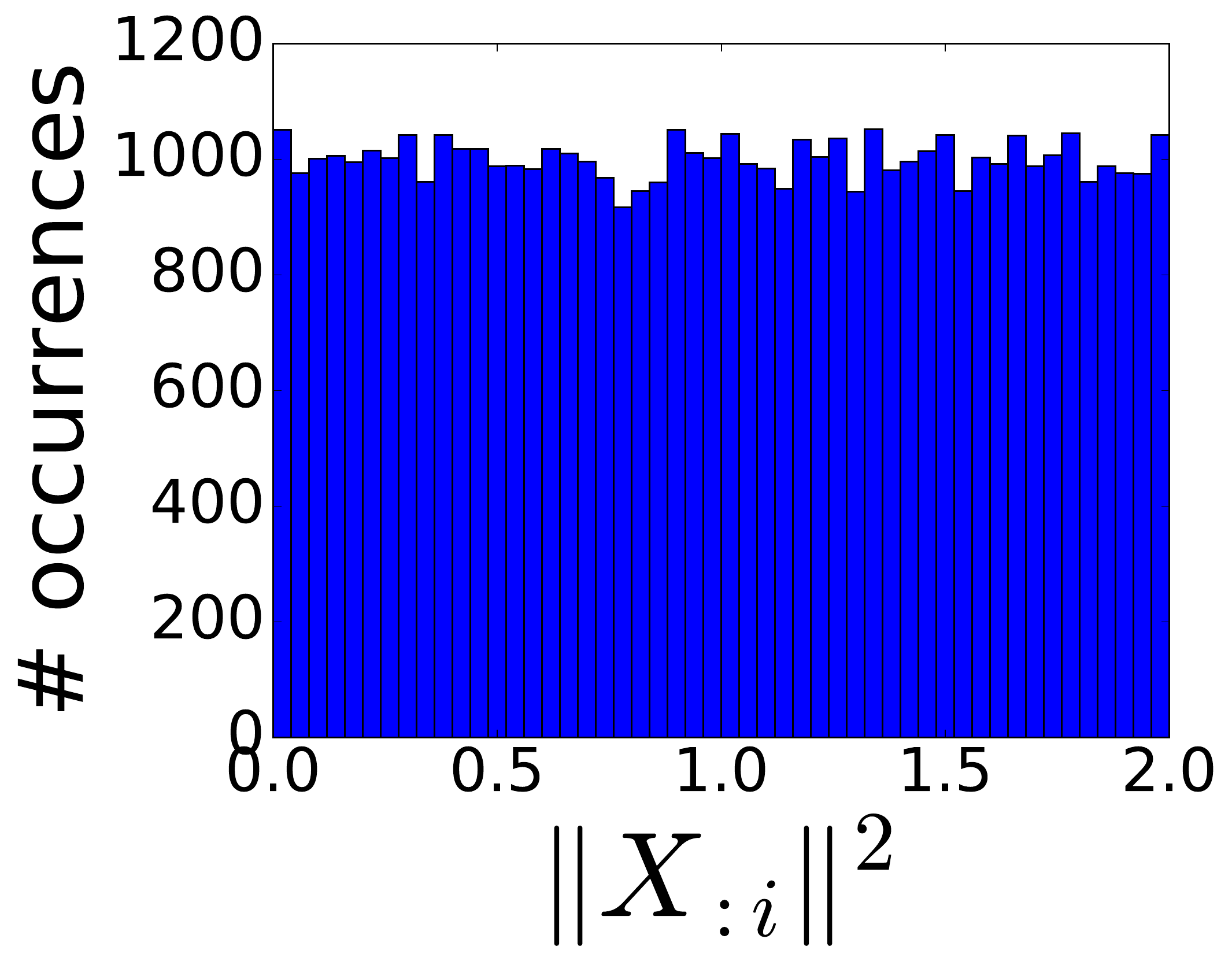}}
\hfill
\subfigure[chisq100]{\includegraphics[width = 0.32\columnwidth]{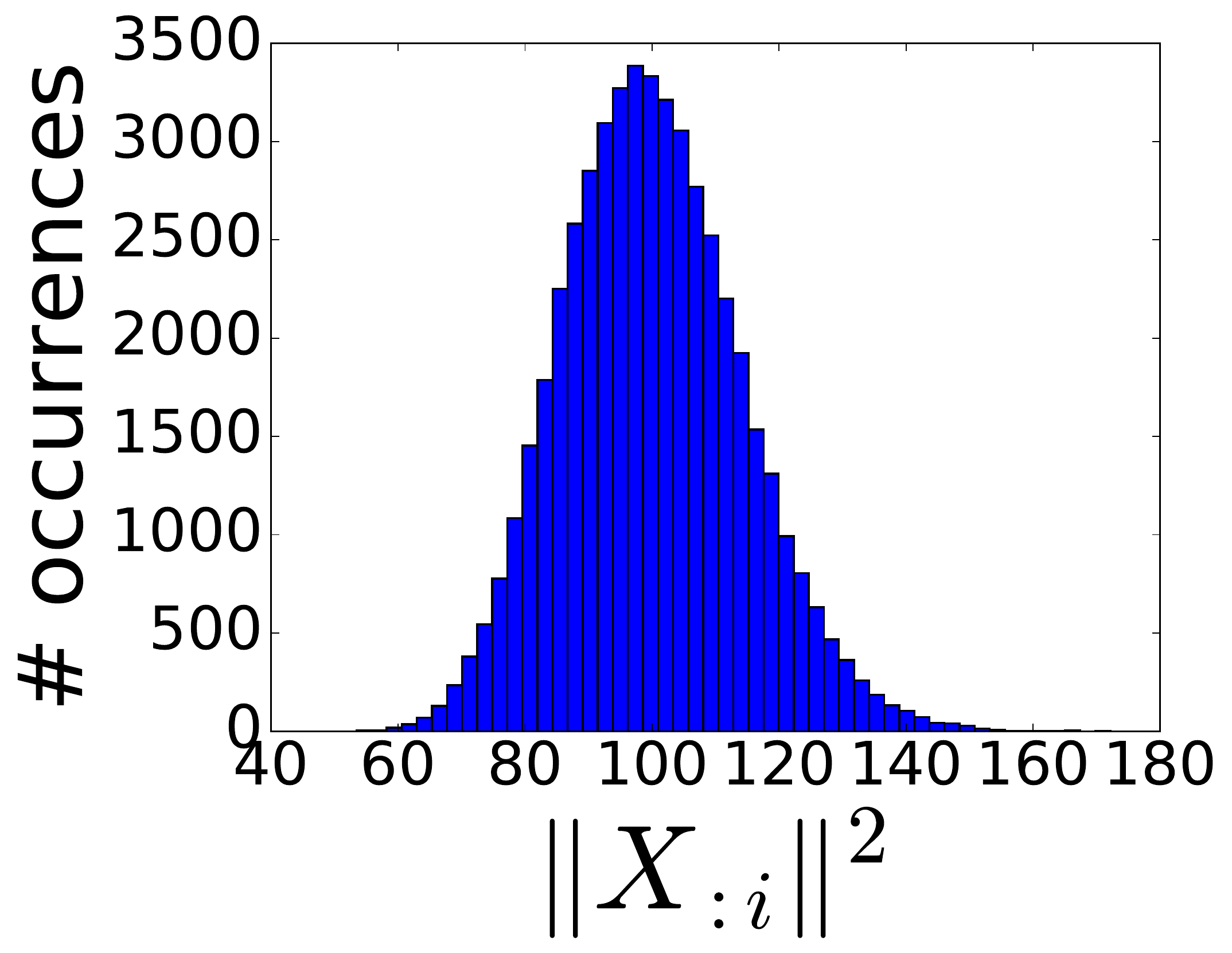}}
\hfill
\subfigure[chisq10]{\includegraphics[width = 0.32\columnwidth]{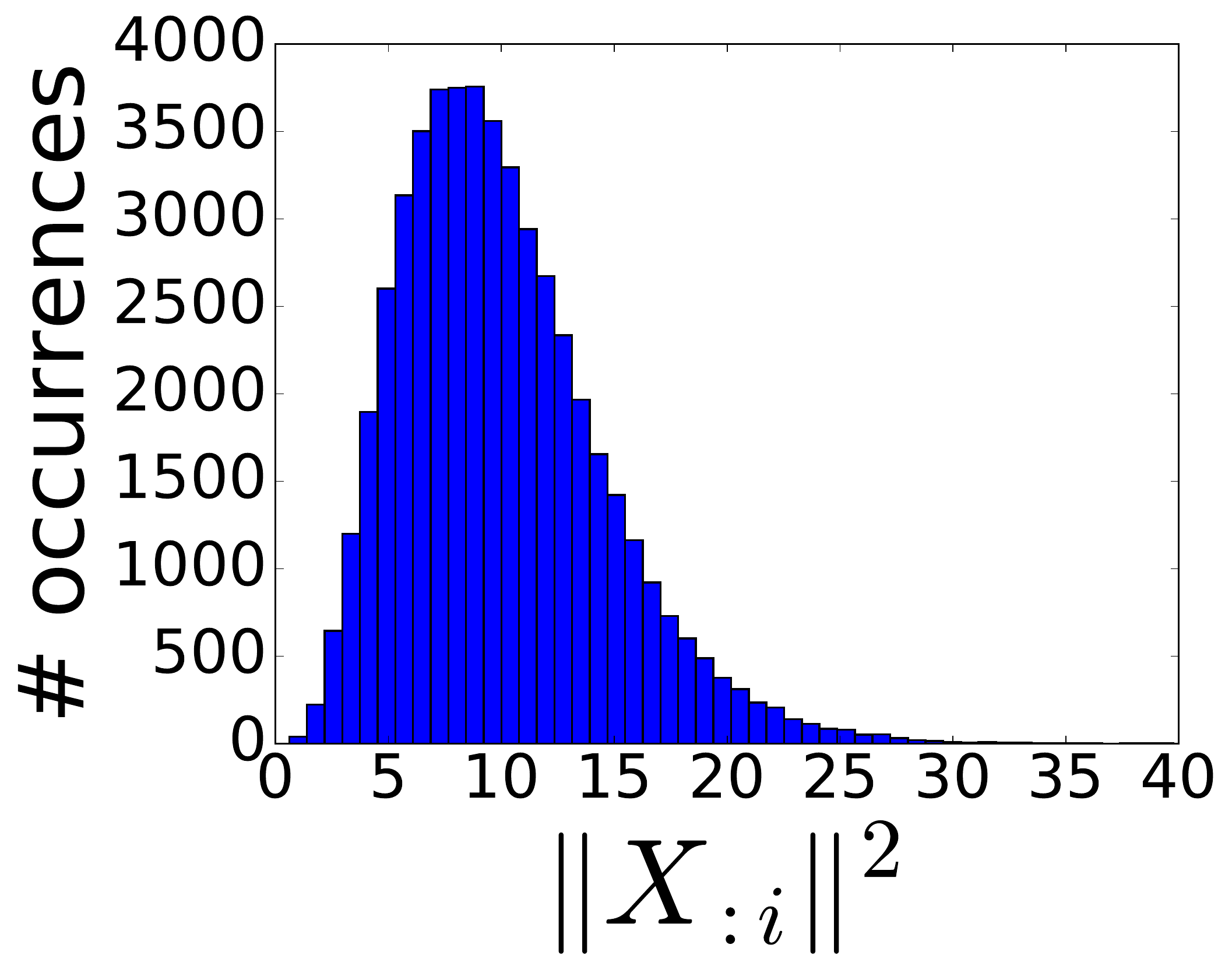}}
\centering{
\subfigure[chisq1]{\includegraphics[width = 0.32\columnwidth]{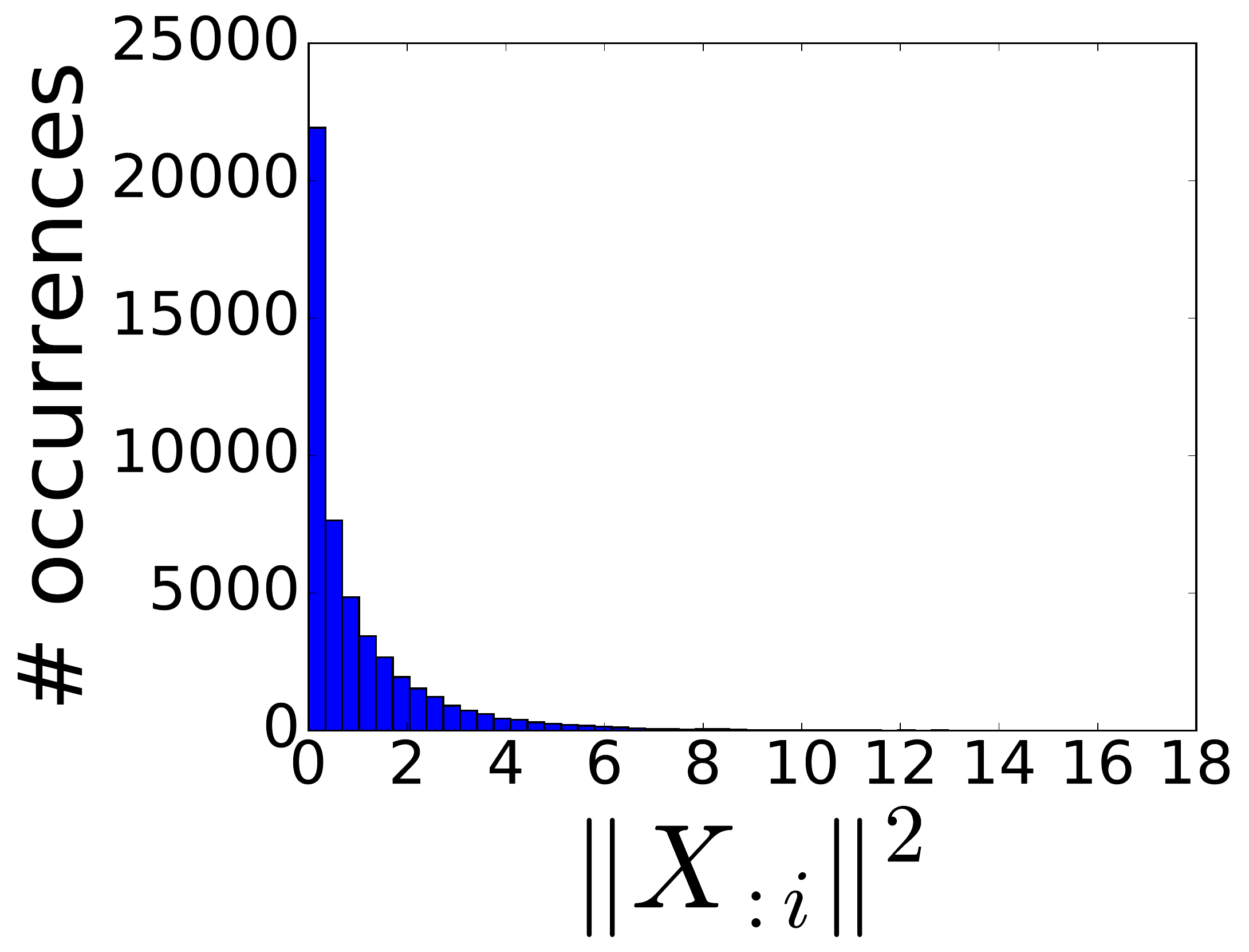}}
\subfigure[extreme]{\includegraphics[width = 0.32\columnwidth]{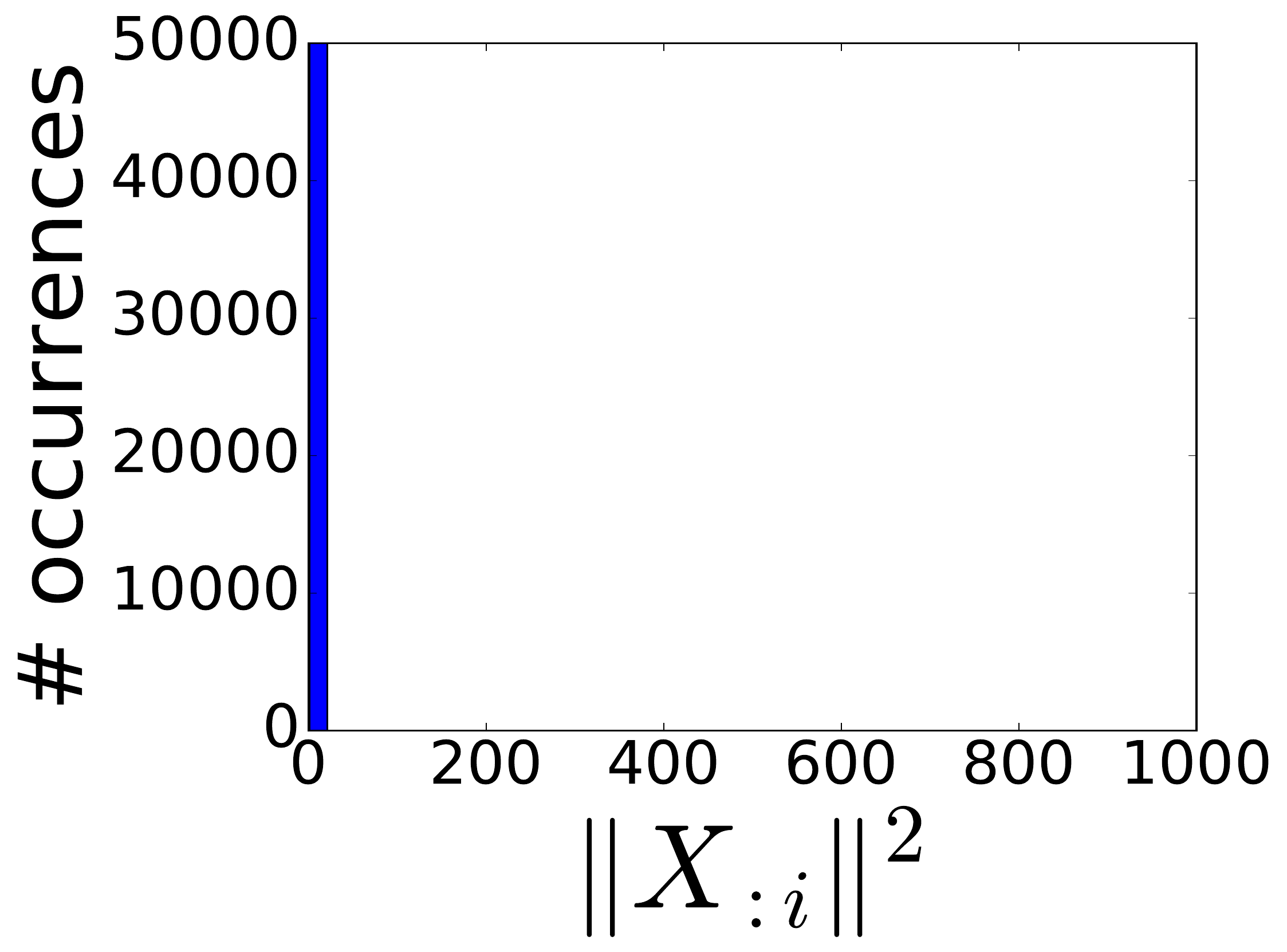}}}
\caption{The distribution of $\|\bX_{:i}\|^2$ for synthetic data}
\label{fig:hists}
\end{figure} 

The corresponding experiments can be found in Figure~\ref{fig:artificial_0.8} and Figure~\ref{fig:artificial_0.1}. The theoretical and empirical speedup are also summarized in Tables~\ref{tab:artificial_data_0.1} and \ref{tab:artificial_data_0.8}.

\begin{table}[!ht]
\centering 
\begin{tabular}{|l|c|c|c|c|c|c|}
\hline
Data & $\tau=1$         & $\tau=2$        & $\tau=4$        & $\tau=8$         & $\tau=16$         & $\tau=32$  \\ \hline
uniform     & 1.2~:~1.0	    & 1.2~:~1.1	   & 1.2~:~1.1    & 1.2~:~1.1     & 1.3~:~1.1      & 1.4~:~1.1    \\ \hline
chisq100    & 1.5~:~1.3	    & 1.5~:~1.3     & 1.5~:~1.4     & 1.6~:~1.4     & 1.6~:~1.4      & 1.6~:~1.4    \\ \hline
chisq10     & 1.9~:~1.4     & 1.9~:~1.5     & 2.0~:~1.4     & 2.2~:~1.5     & 2.5~:~1.6      & 2.8~:~1.7 \\ \hline
chisq1     	& 1.9~:~1.4     & 2.0~:~1.4     & 2.2~:~1.5     & 2.5~:~1.6     & 3.1~:~1.6      & 4.2~:~1.7  \\ \hline
extreme	    & 8.8~:~4.8	    & 9.6~:~6.6	   & 11~:~6.4	  & 14~:~6.4     & 20~:~6.9      & 32~:~6.1 \\ \hline
\end{tabular}
\caption{ The \textbf{theoretical} ~:~ \textbf{empirical} ratios $\theta^{(\tau\text{-imp})}/\theta^{(\tau\text{-nice})}$ for sparse artificial data ($\omega' = 0.1$)}
\label{tab:artificial_data_0.1}
\end{table}

\begin{table}[!ht]

\centering 

\begin{tabular}{|l|c|c|c|c|c|c|}
\hline
Data & $\tau=1$         & $\tau=2$        & $\tau=4$        & $\tau=8$         & $\tau=16$         & $\tau=32$ \\ \hline
uniform     & 1.2~:~1.1	    & 1.2~:~1.1	   & 1.4~:~1.2    & 1.5~:~1.2     & 1.7~:~1.3      & 1.8~:~1.3    \\ \hline
chisq100    & 1.5~:~1.3	    & 1.6~:~1.4     & 1.6~:~1.5     & 1.7~:~1.5     & 1.7~:~1.6      & 1.7~:~1.6    \\ \hline
chisq10     & 1.9~:~1.3     & 2.2~:~1.6     & 2.7~:~2.1     & 3.1~:~2.3     & 3.5~:~2.5      & 3.6~:~2.7 \\ \hline
chisq1     	& 1.9~:~1.3     & 2.6~:~1.8     & 3.7~:~2.3     & 5.6~:~2.9     & 7.9~:~3.2      & 10~:~3.9  \\ \hline
extreme	    & 8.8~:~5.0	    & 15~:~7.8	   & 27~:~12	  & 50~:~16     & 91~:~21      & 154~:~28 \\ \hline
\end{tabular}
\caption{The \textbf{theoretical}~:~\textbf{empirical} ratios $\theta^{(\tau\text{-imp})}/\theta^{(\tau\text{-nice})}$. Artificial data with $\omega' = 0.8$ (dense)}
\label{tab:artificial_data_0.8}
\end{table}

\subsection{Real data}

We used several publicly available datasets\footnote{https://www.csie.ntu.edu.tw/~cjlin/libsvmtools/datasets/}, summarized in Table~\ref{tab:datasets}.  Experimental results are in Figure~\ref{fig:real}. The theoretical and empirical speedup table for these datasets  can be found in Table~\ref{tab:real_data}.

\begin{table}[!ht]
\centering 
\begin{tabular}{|l|l|l|l|l|}
\hline
Dataset     & \#samples & \#features & sparsity & $\sigma$ \\ \hline
ijcnn1	    & 35,000	& 23		 & 60.1\%	 & 2.04 \\ \hline
protein     & 17,766    & 358        & 29.1\%  & 1.82   \\ \hline
w8a         & 49,749    & 301        & 4.2\%   & 9.09 \\ \hline
url			& 2,396,130 & 3,231,962 & 0.04 \% & 4.83 \\ \hline
aloi        & 108,000	& 129		 & 24.6\%  & 26.01    \\ \hline
\end{tabular}
\caption{ Summary of real data sets ($\sigma$ = predicted speedup).}
\label{tab:datasets}
\end{table}

\begin{table}[!ht]
\centering 
\begin{tabular}{|l|c|c|c|c|c|c|}
\hline
 Data & $\tau=1$         & $\tau=2$        & $\tau=4$        & $\tau=8$         & $\tau=16$         & $\tau=32$ \\ \hline
ijcnn1	    & 1.2~:~1.1	    & 1.4~:~1.1	   & 1.6~:~1.3	  & 1.9~:~1.6     & 2.2~:~1.6      & 2.3~:~1.8 \\ \hline
protein     & 1.3~:~1.2     & 1.4~:~1.2     & 1.5~:~1.4     & 1.7~:~1.4     & 1.8~:~1.5      & 1.9~:~1.5 \\ \hline
w8a         & 2.8~:~2.0     & 2.9~:~1.9     & 2.9~:~1.9     & 3.0~:~1.9     & 3.0~:~1.8      & 3.0~:~1.8 \\ \hline
url         & 3.0~:~2.3    & 2.6~:~2.1     & 2.0~:~1.8     & 1.7~:~1.6     & 1.8~:~1.6      & 1.8~:~1.7    \\ \hline
aloi        & 13~:~7.8	    & 12~:~8.0	   & 11~:~7.7    & 9.9~:~7.4     & 9.3~:~7.0      & 8.8~:~6.7    \\ \hline
\end{tabular}
\caption{The \textbf{theoretical} ~:~\textbf{empirical} ratios $\theta^{(\tau\text{-imp})}/\theta^{(\tau\text{-nice})}$.}
\label{tab:real_data}
\end{table}

\subsection{Conclusion}

In all experiments, $\tau$-importance sampling performs significantly better than $\tau$-nice sampling. The theoretical speedup factor computed by $\theta^{(\tau\text{-imp})}/\theta^{(\tau\text{-nice})}$ provides an excellent estimate of the actual speedup. We can observe that on denser data the speedup is higher than on sparse data. This matches the theoretical intuition for $v_i$ for both samplings. As we observed for artificial data,  for extreme datasets  the speedup can be arbitrary large, even several orders of magnitude. {\em A rule of thumb: if one has data with large $\sigma$, practical speedup from using importance minibatch sampling   will likely be dramatic.}

\begin{figure}[!h]
\subfigure[ijcnn1, $\tau$-nice (left), $\tau$-importance (right)]{\includegraphics[width=0.25\columnwidth]{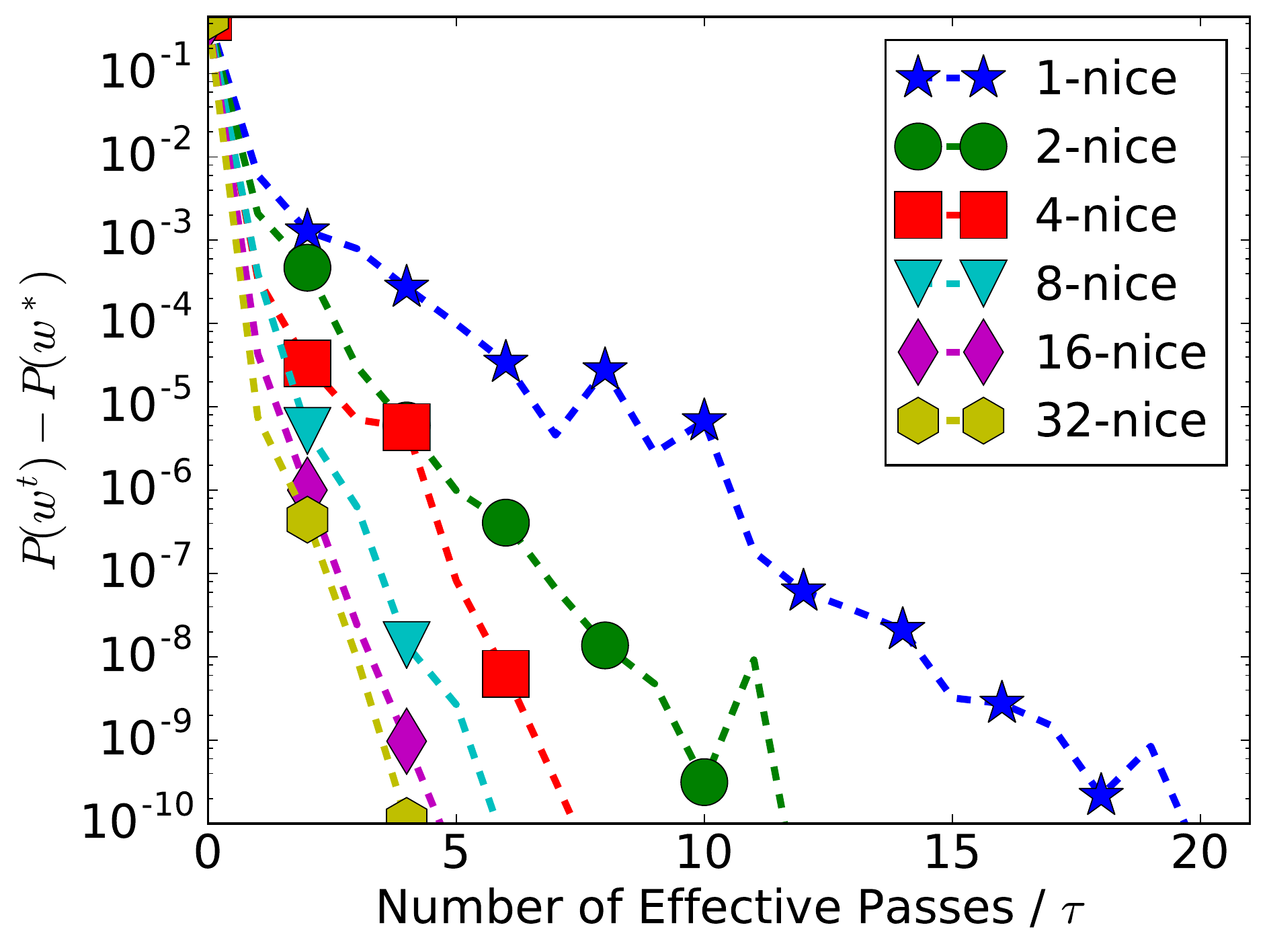}\includegraphics[width=0.25\columnwidth]{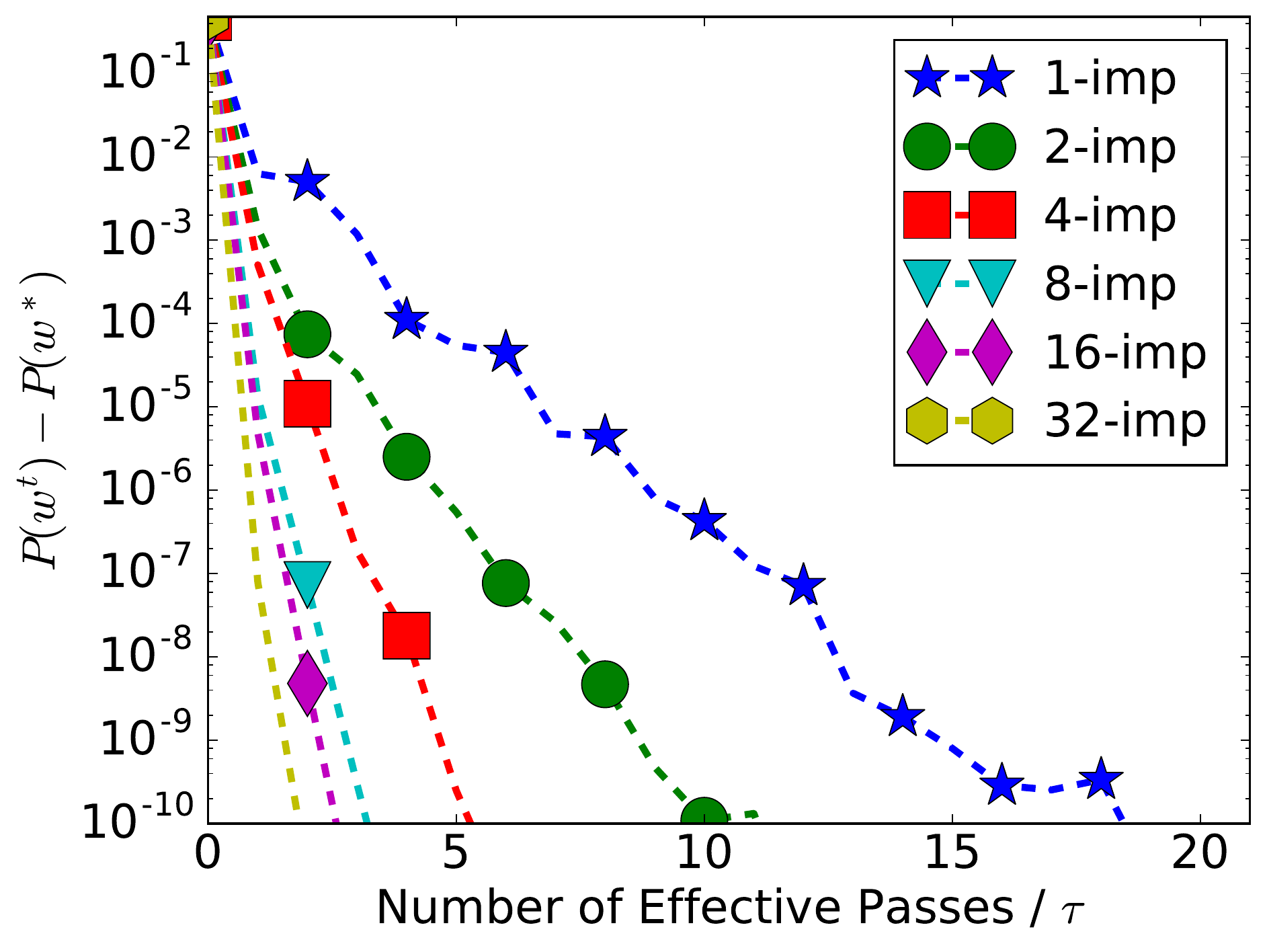}}
\hfill
\subfigure[protein, $\tau$-nice (left), $\tau$-importance (right)]{\includegraphics[width=0.25\columnwidth]{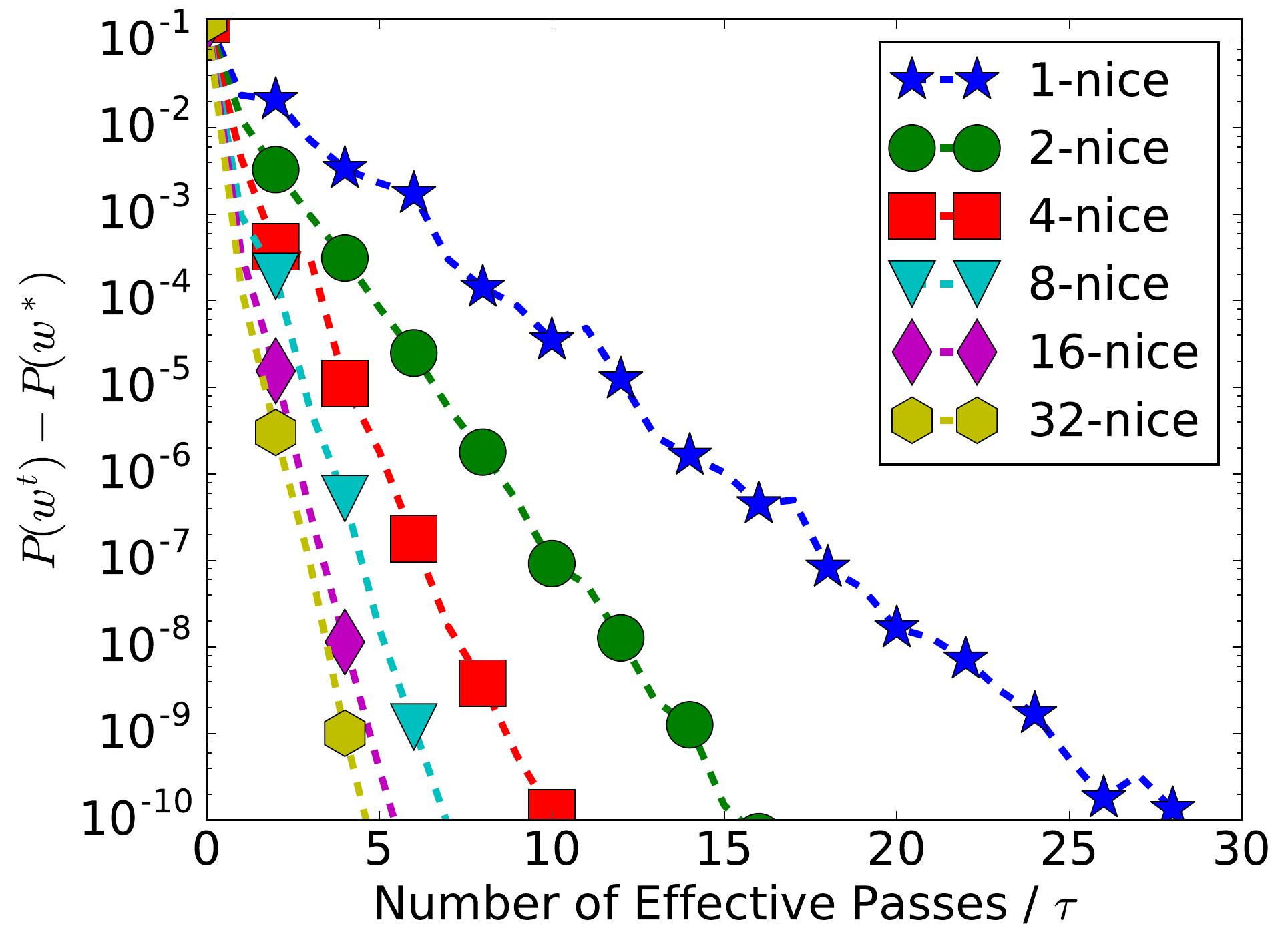}\includegraphics[width=0.25\columnwidth]{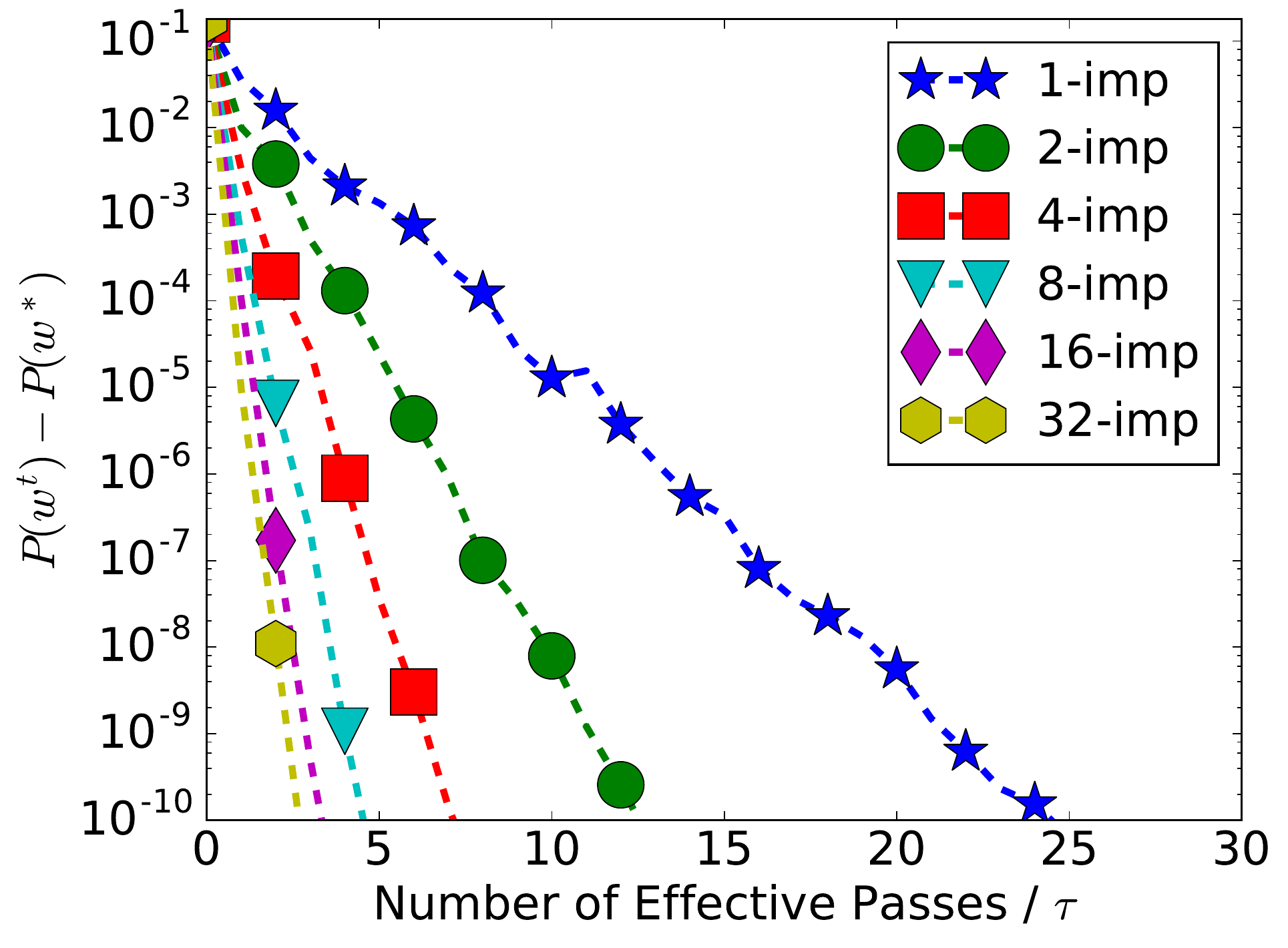}}
\hfill
\subfigure[w8a, $\tau$-nice (left), $\tau$-importance (right)]{\includegraphics[width=0.25\columnwidth]{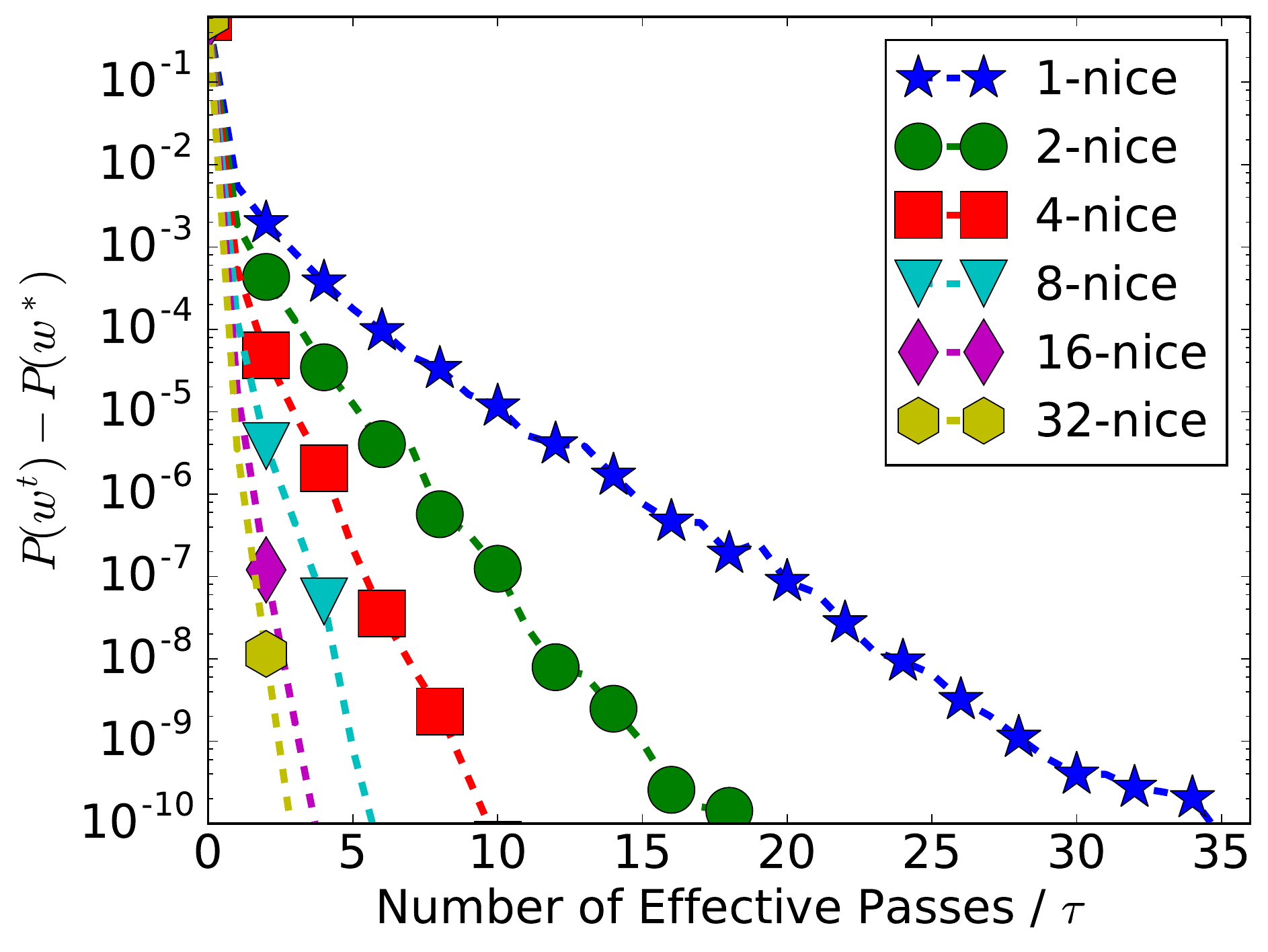}\includegraphics[width=0.25\columnwidth]{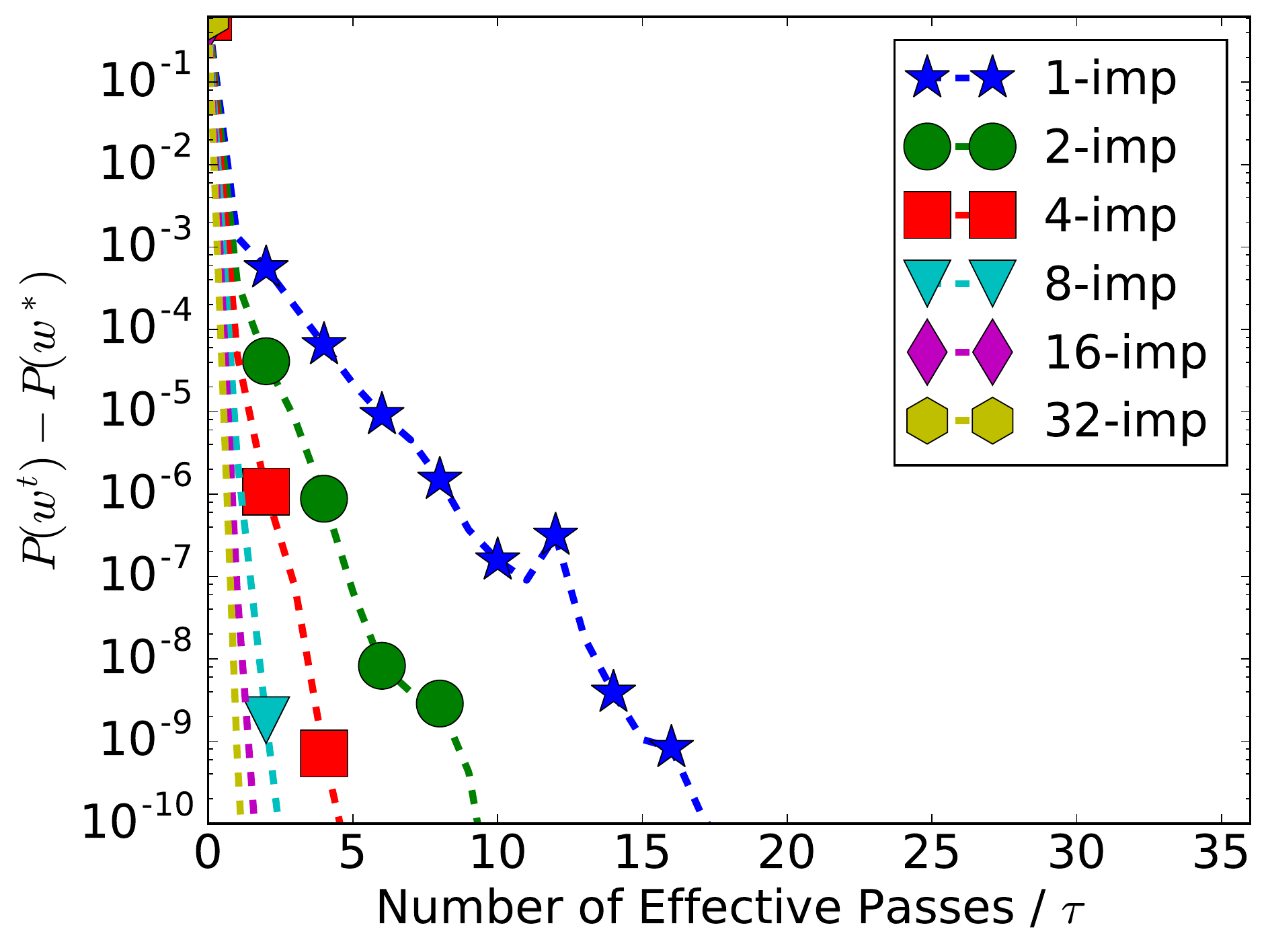}}
\hfill
\subfigure[aloi, $\tau$-nice (left), $\tau$-importance (right)]{\includegraphics[width=0.25\columnwidth]{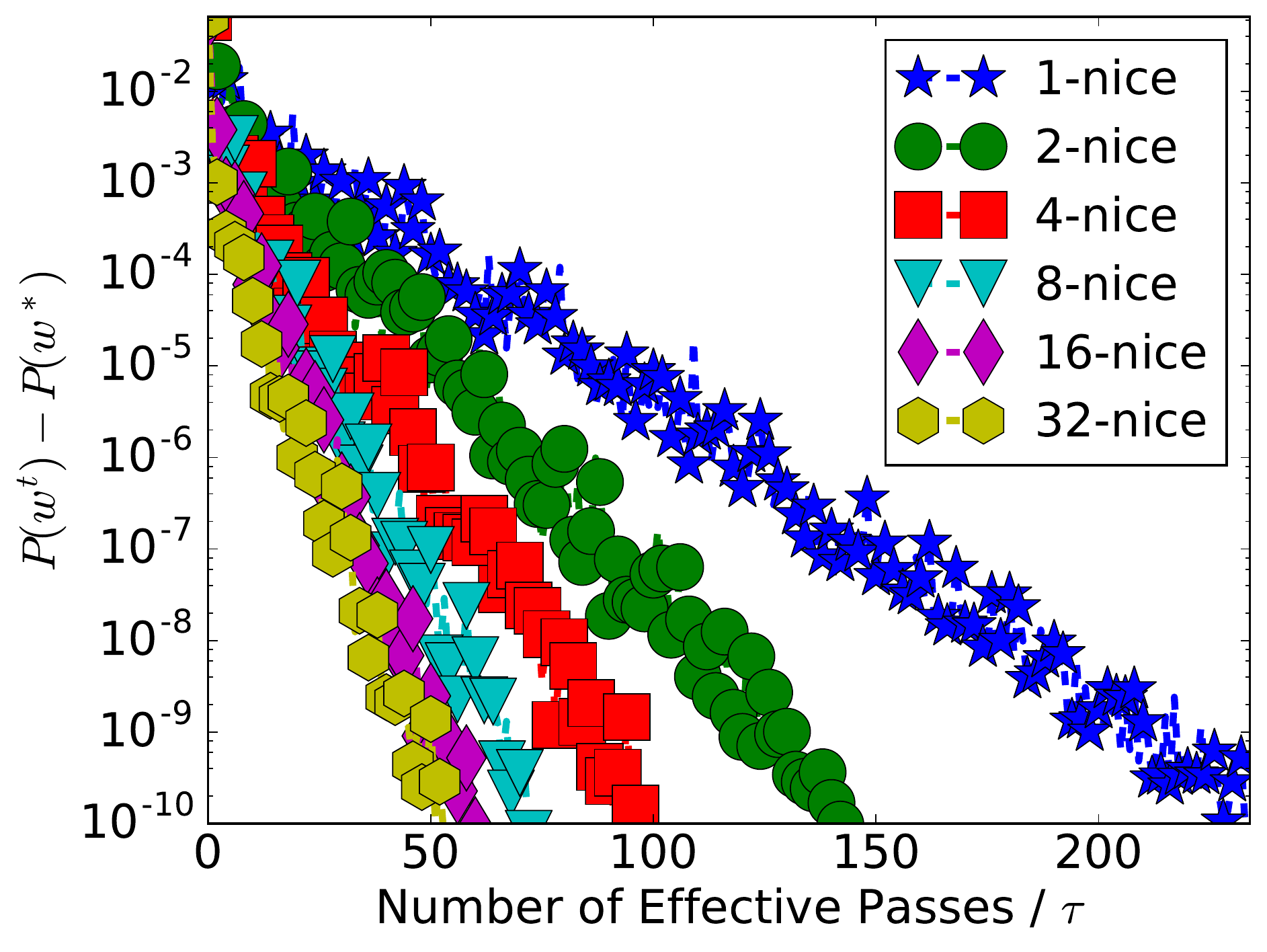}\includegraphics[width=0.25\columnwidth]{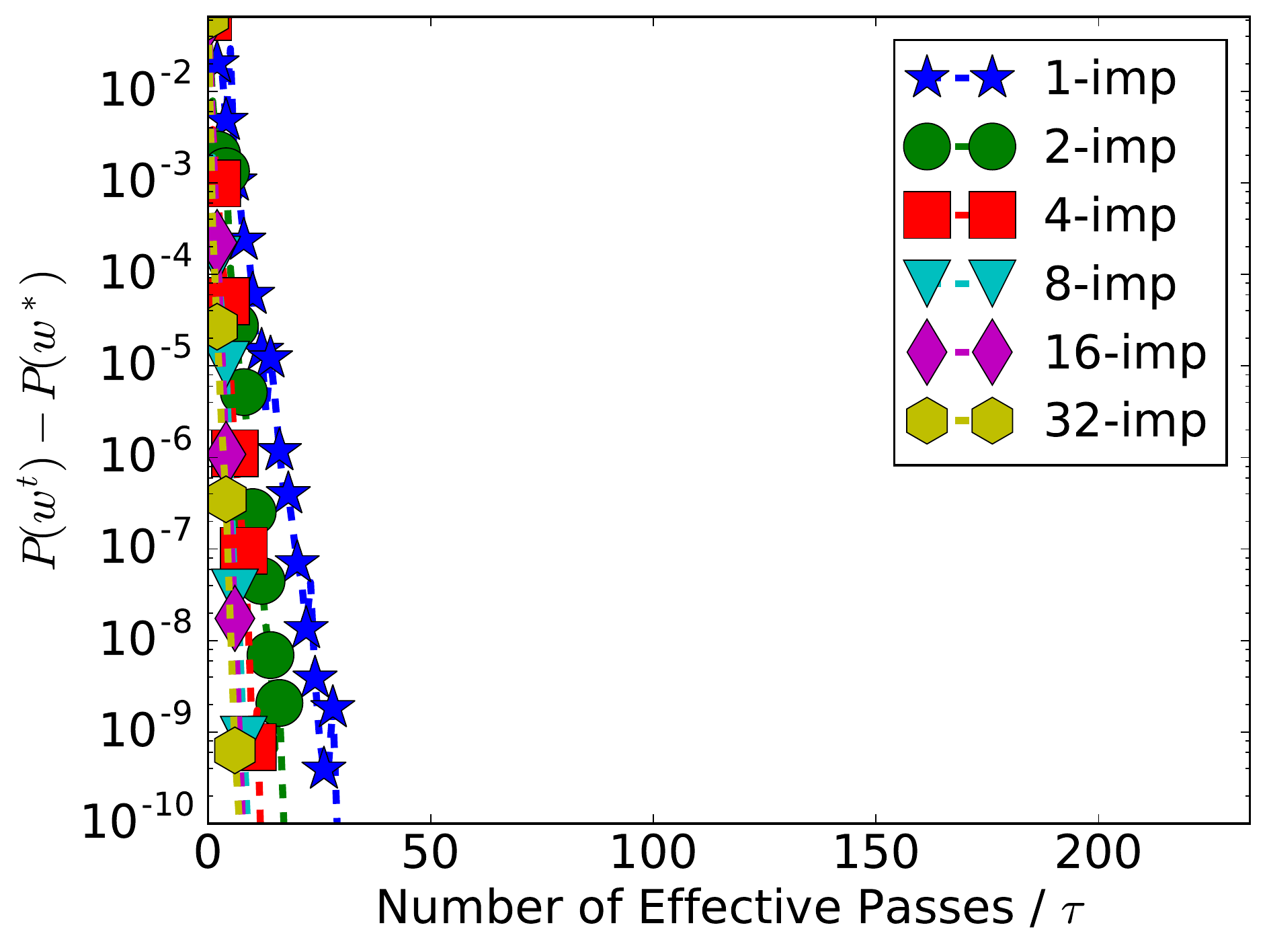}}
\center{
\subfigure[url, $\tau$-nice (left), $\tau$-importance (right)]{\includegraphics[width=0.25\columnwidth]{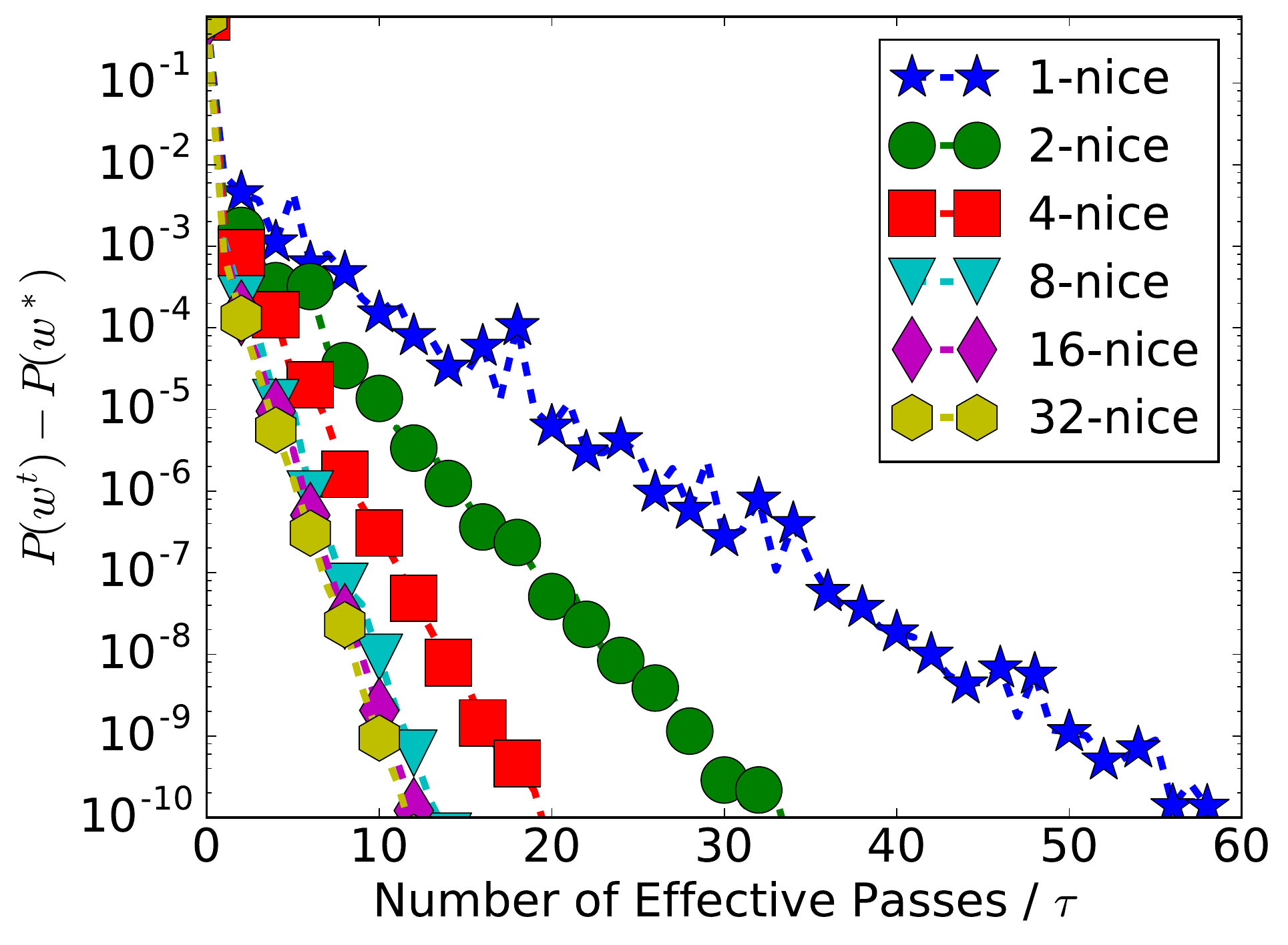}\includegraphics[width=0.25\columnwidth]{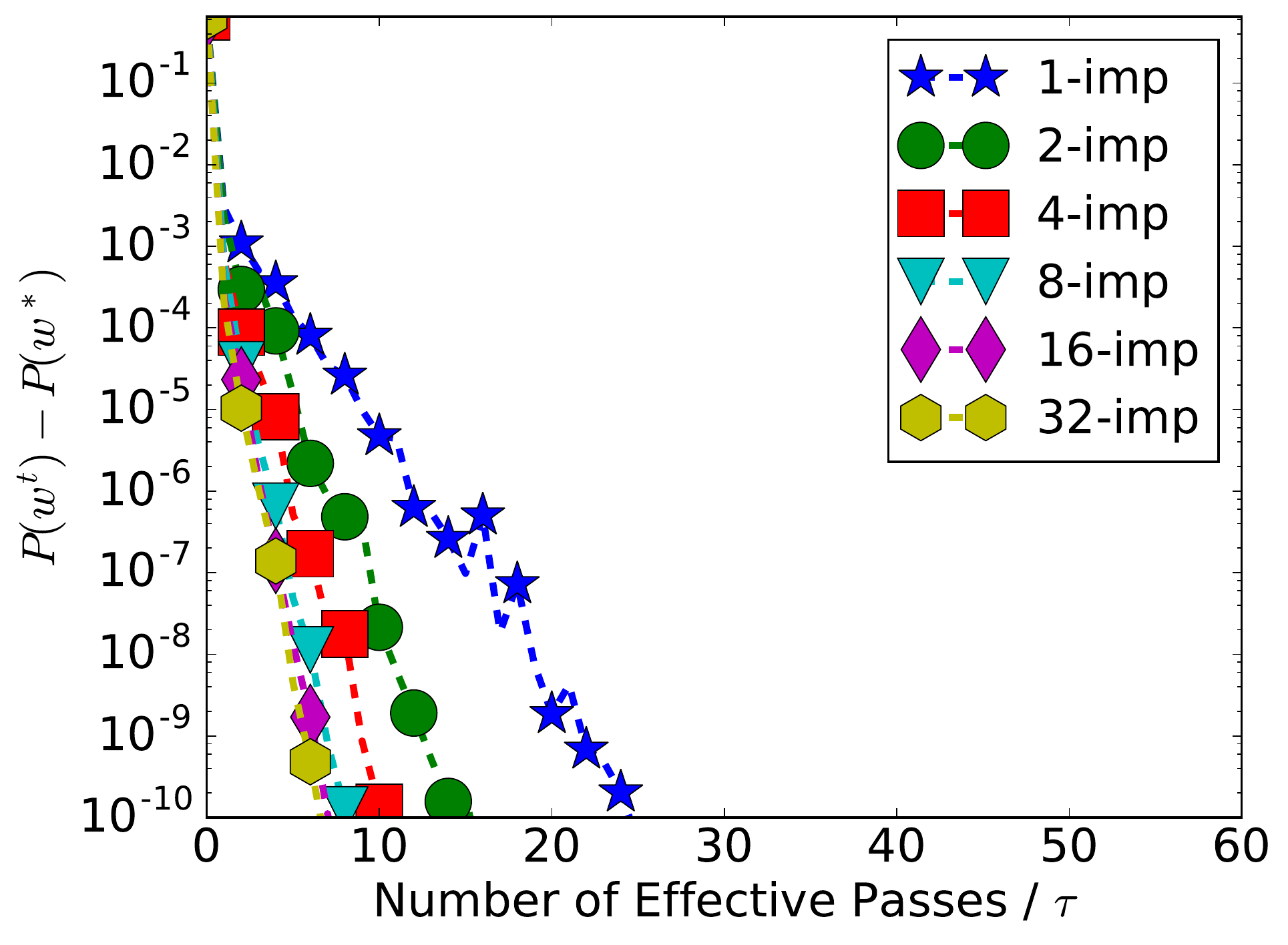}}
\caption{Real Datasets from Table \ref{tab:datasets}}}
\label{fig:real}
\end{figure}

\begin{figure}
\begin{minipage}{0.5\textwidth}
\hfill
\subfigure[uniform, $\tau$-nice (left), $\tau$-importance (right)]{\includegraphics[width=0.5\columnwidth]{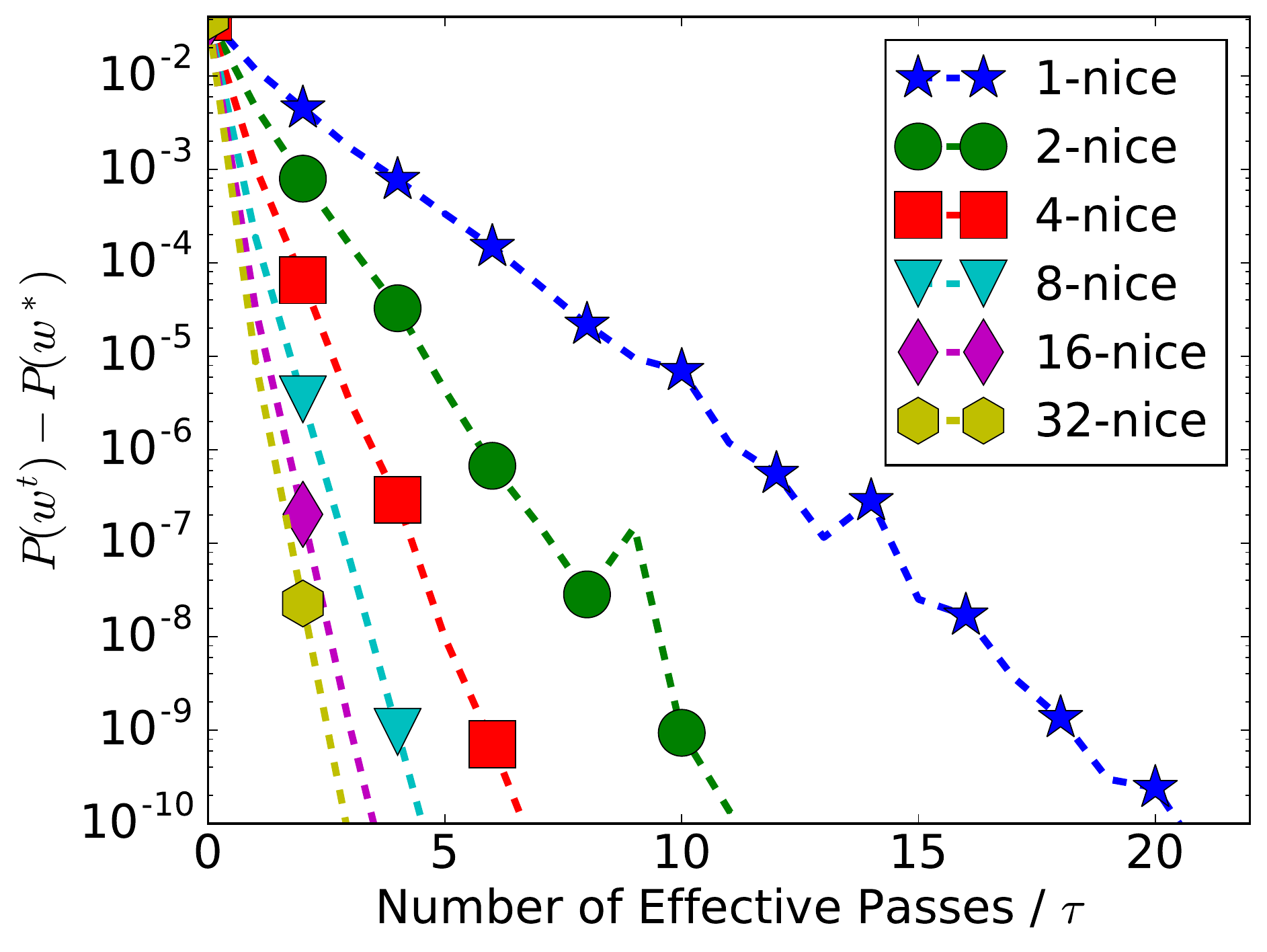}\includegraphics[width=0.5\columnwidth]{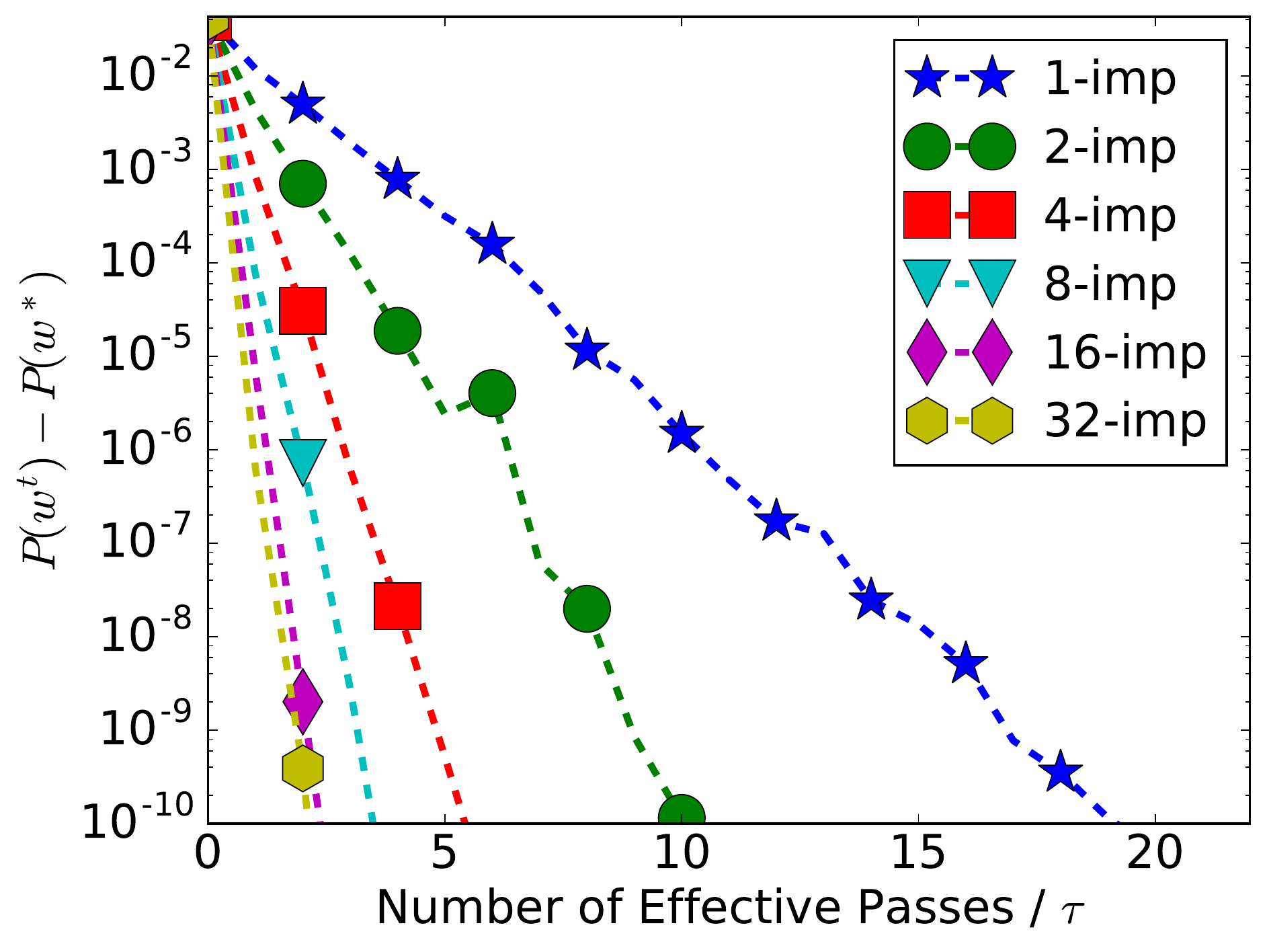}}
\hfill
\subfigure[chisq100, $\tau$-nice (left), $\tau$-importance (right)]{\includegraphics[width=0.5\columnwidth]{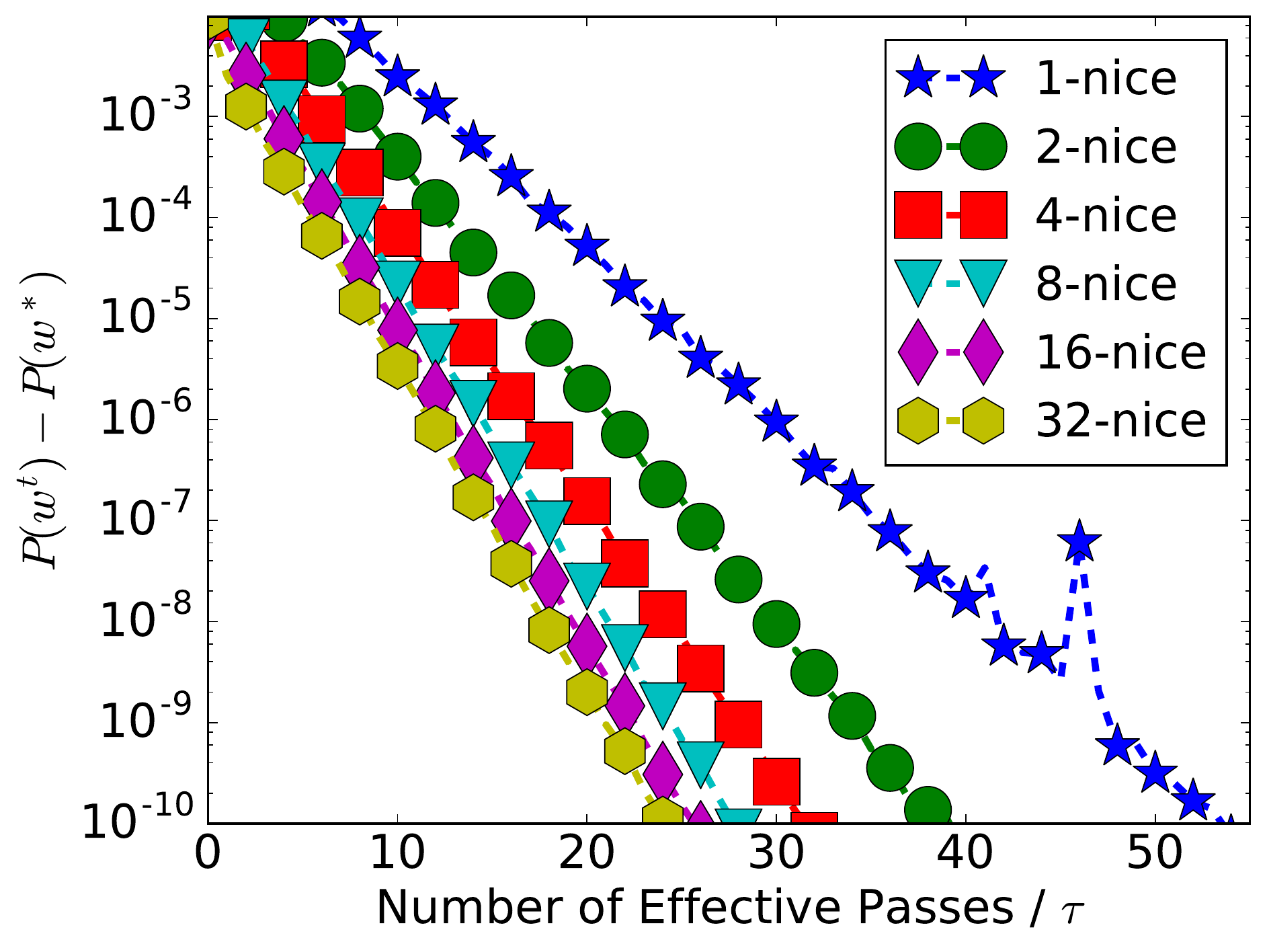}\includegraphics[width=0.5\columnwidth]{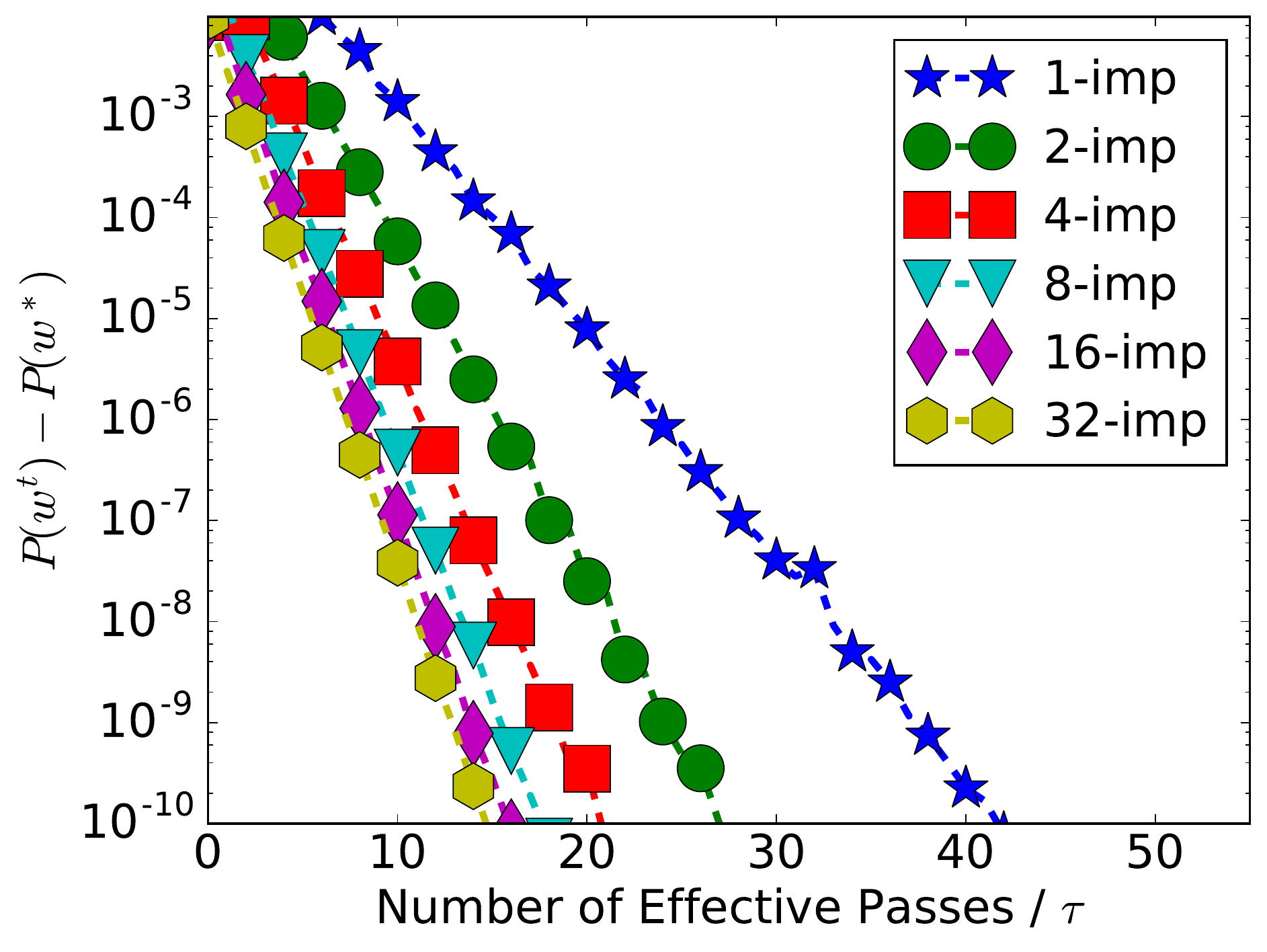}}
\hfill
\subfigure[chisq10, $\tau$-nice (left), $\tau$-importance (right)]{\includegraphics[width=0.5\columnwidth]{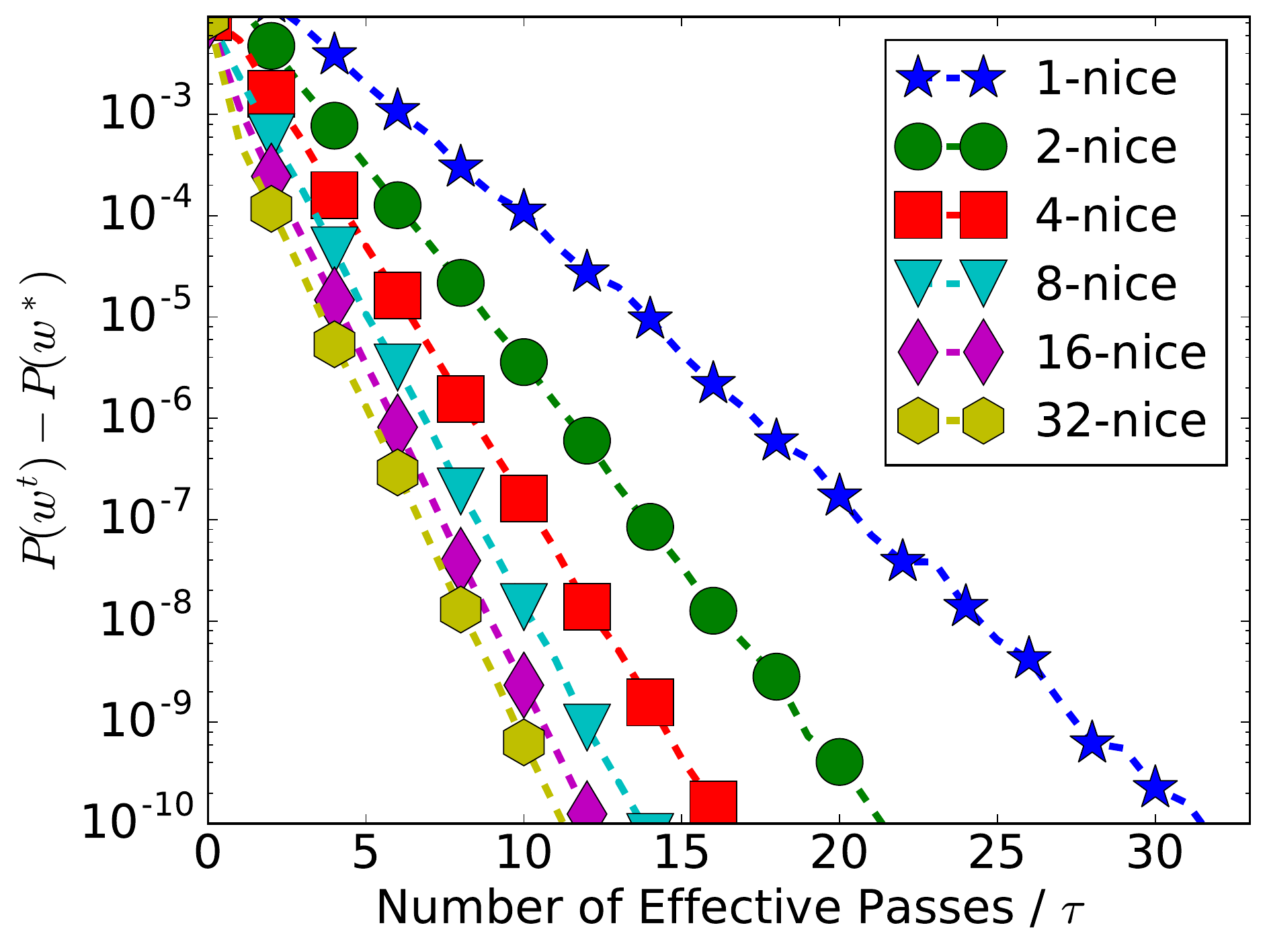}\includegraphics[width=0.5\columnwidth]{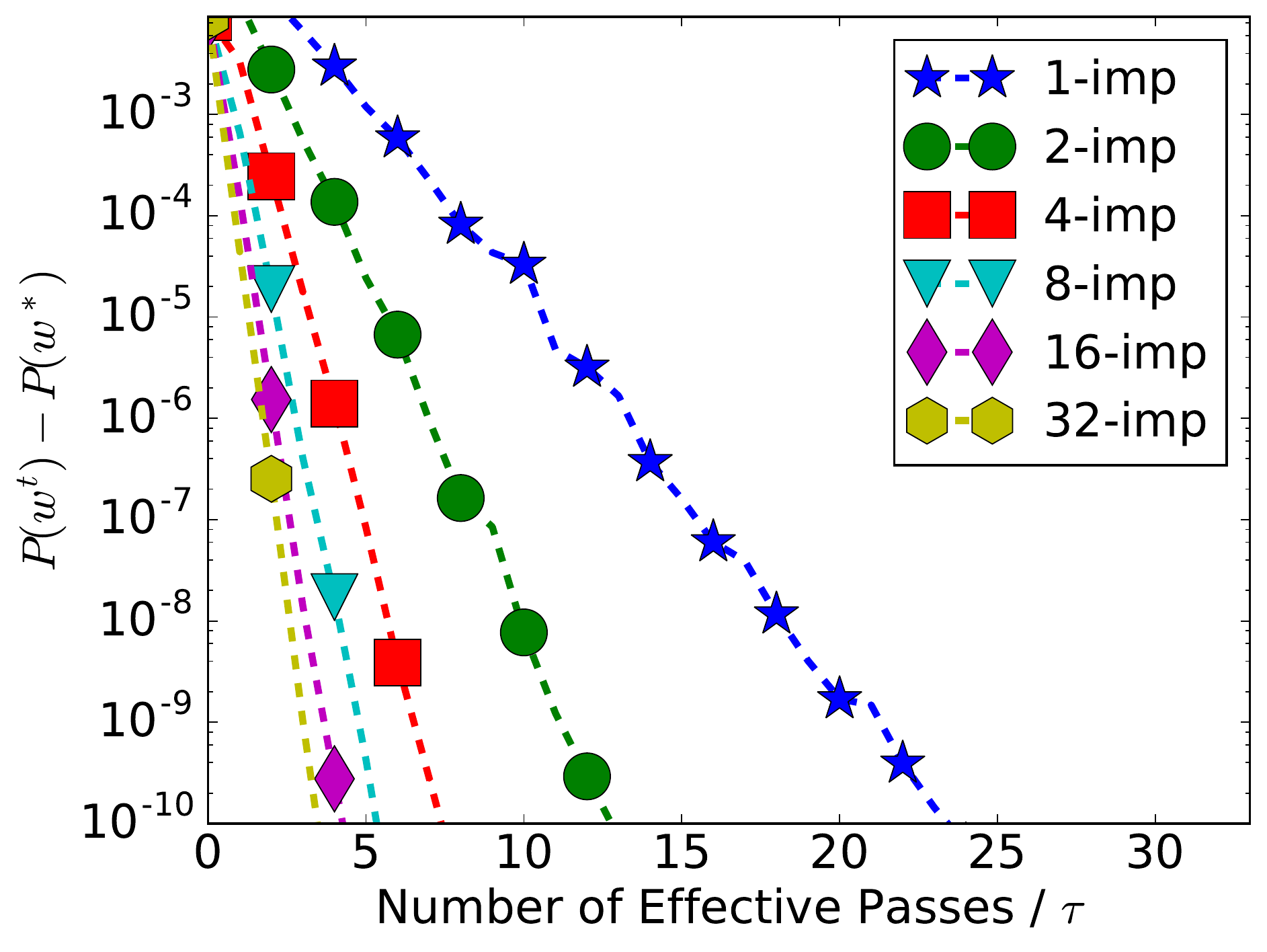}}
\hfill
\subfigure[chisq1, $\tau$-nice (left), $\tau$-importance (right)]{\includegraphics[width=0.5\columnwidth]{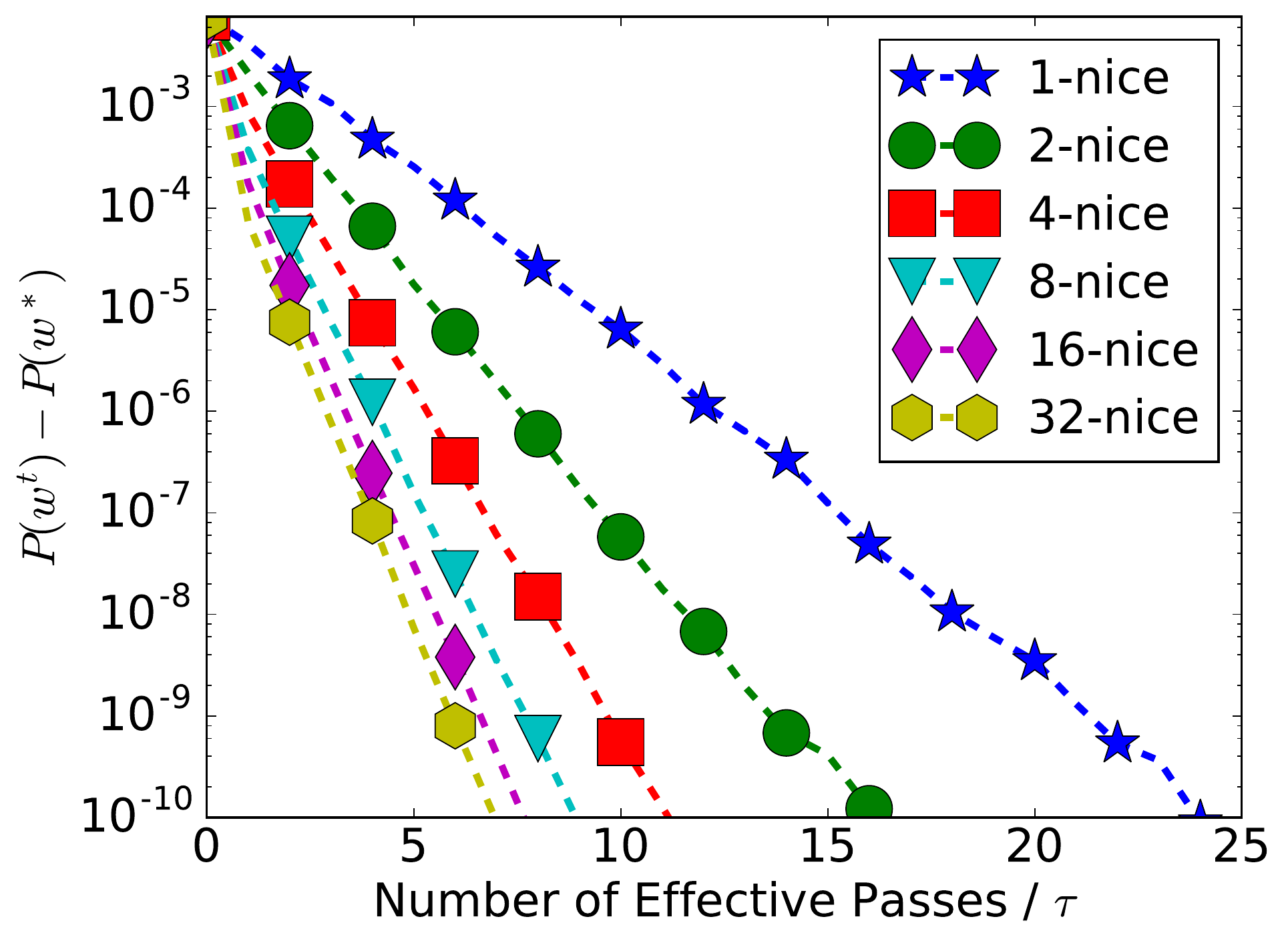}\includegraphics[width=0.5\columnwidth]{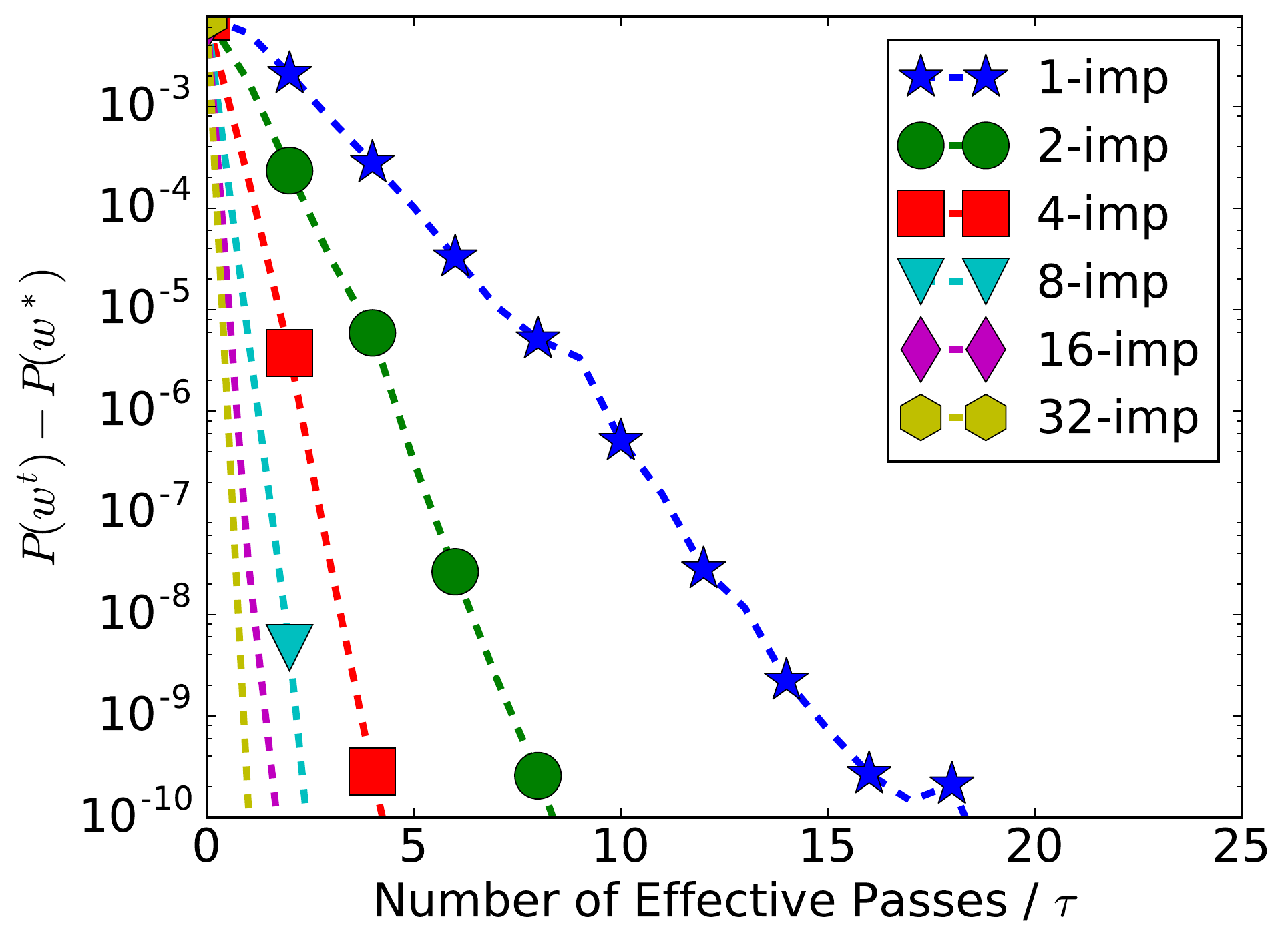}}
\hfill
\subfigure[extreme, $\tau$-nice (left), $\tau$-importance (right)]{\includegraphics[width=0.5\columnwidth]{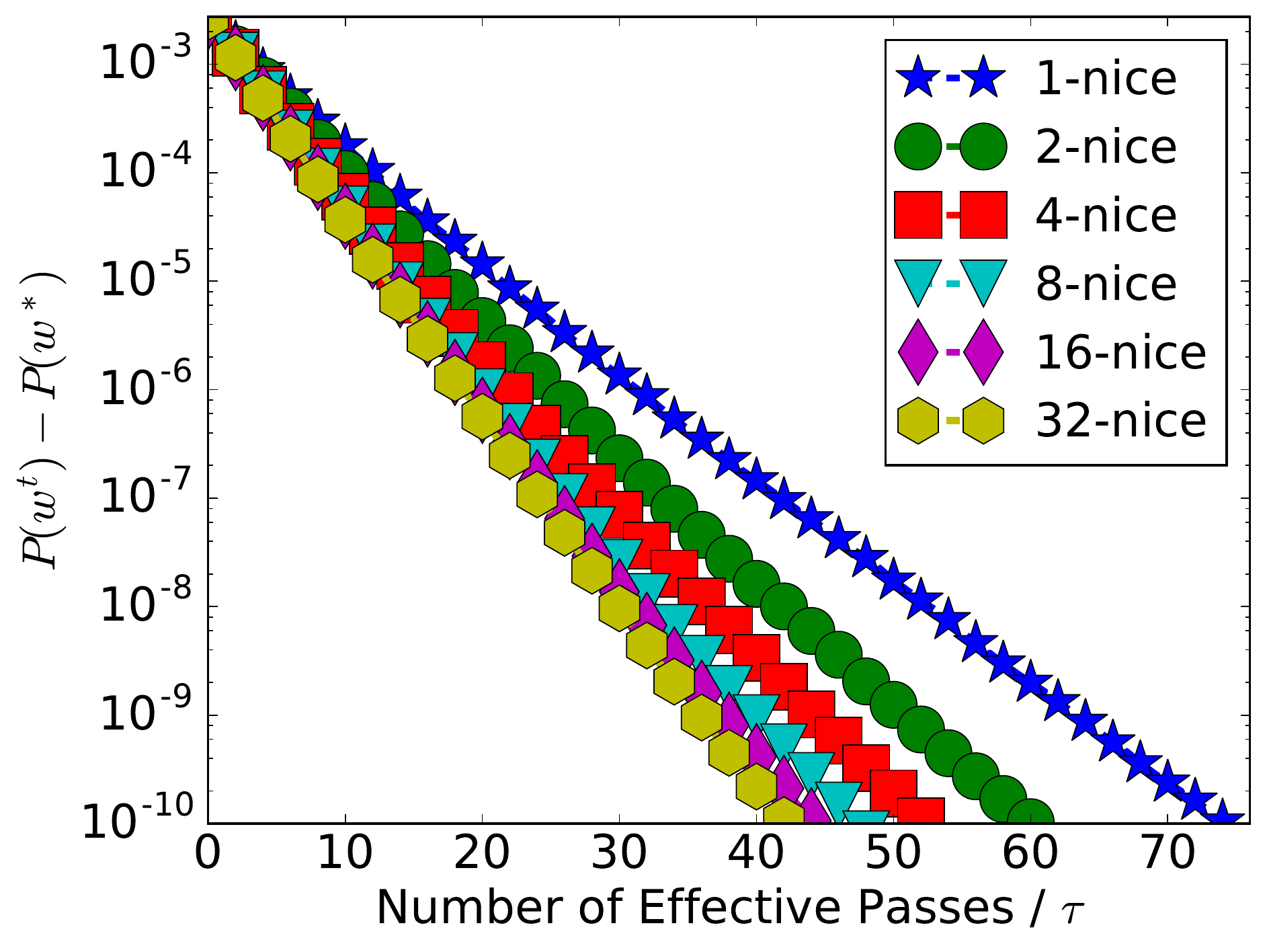}\includegraphics[width=0.5\columnwidth]{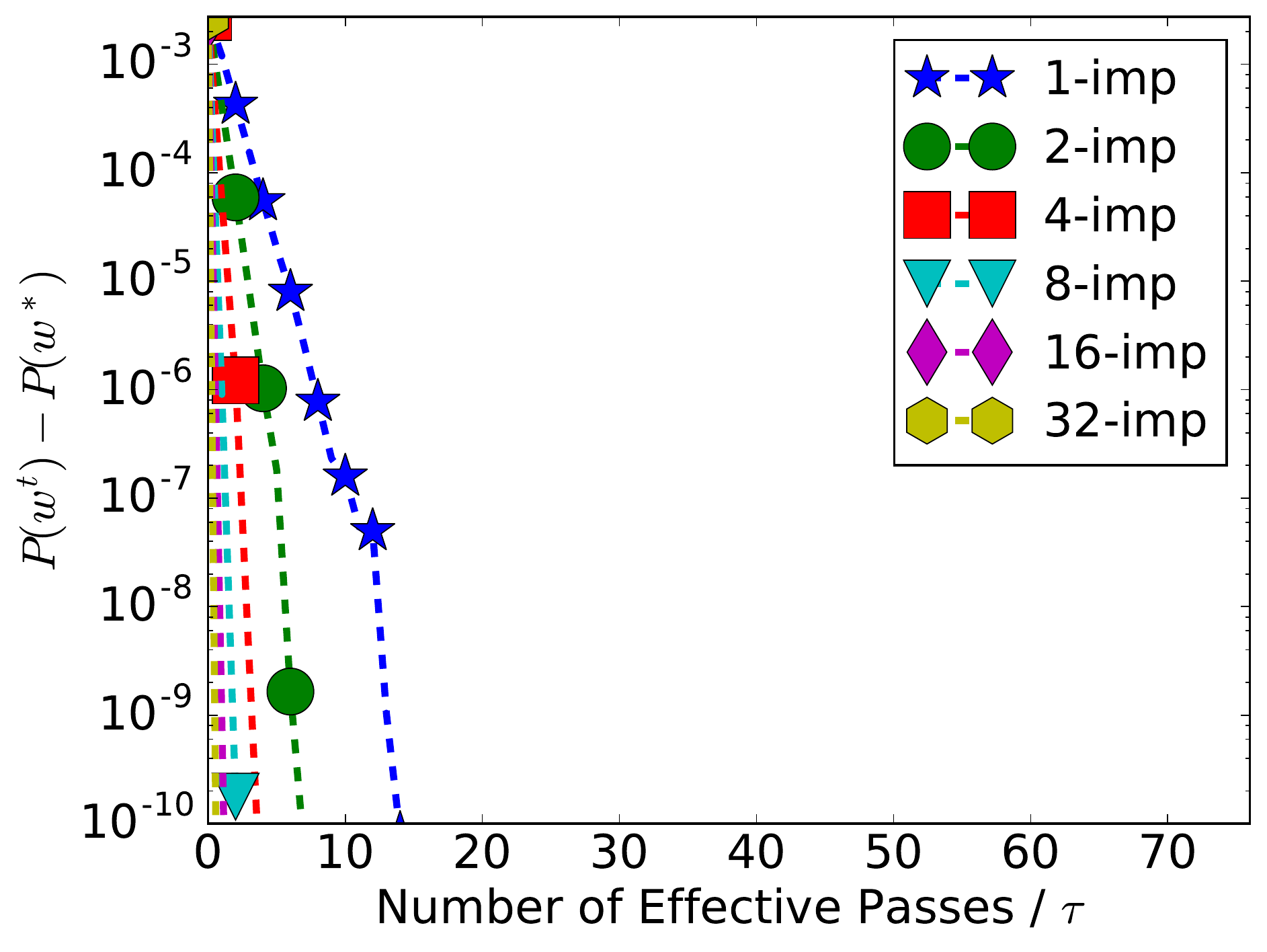}}
\caption{Artificial datasets from Table \ref{tab:dists} with $\omega = 0.8$}
\label{fig:artificial_0.8}
\end{minipage}~
\begin{minipage}{0.5\textwidth}
\hfill
\subfigure[uniform, $\tau$-nice (left), $\tau$-importance (right)]{\includegraphics[width=0.5\columnwidth]{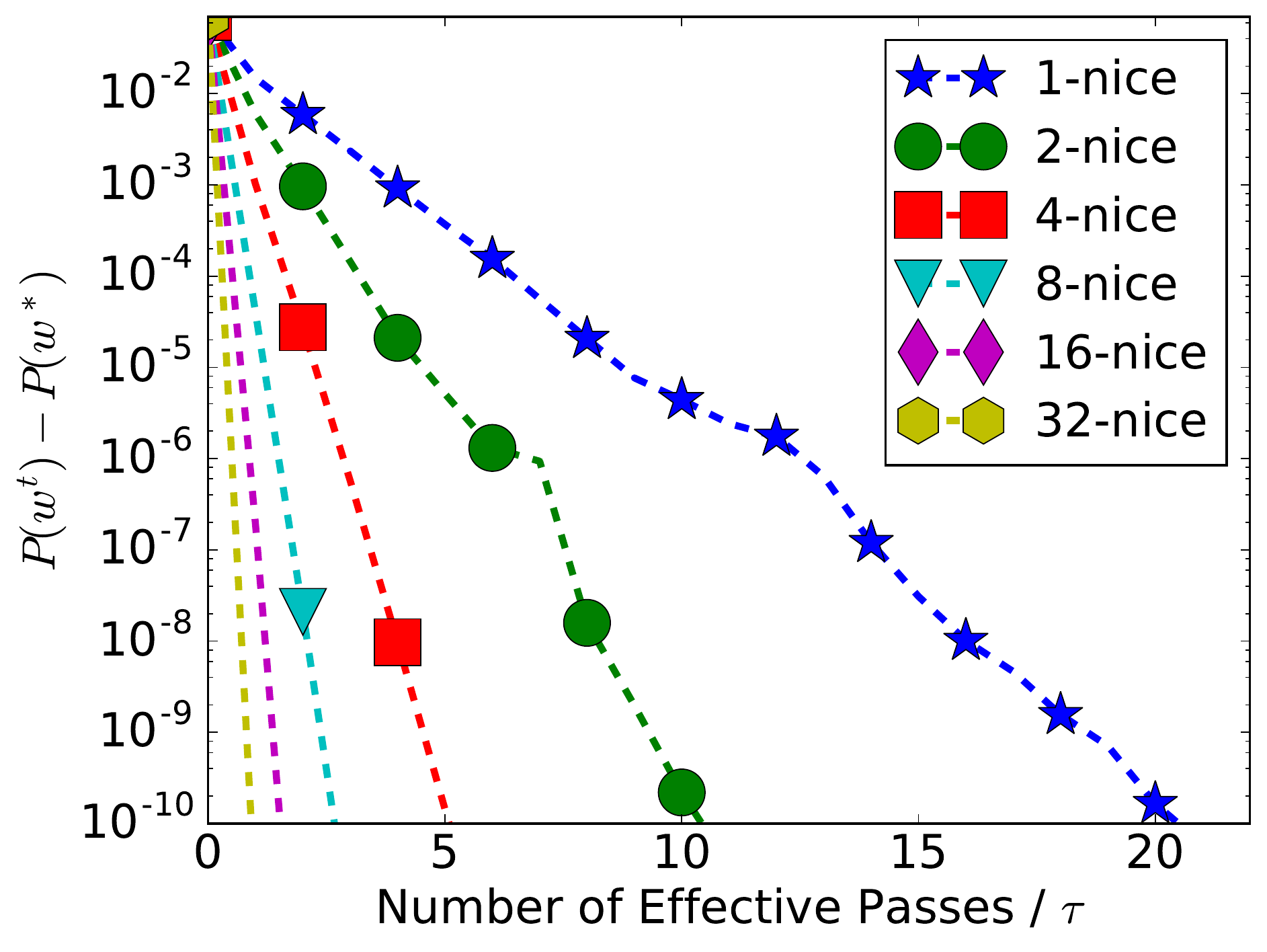}\includegraphics[width=0.5\columnwidth]{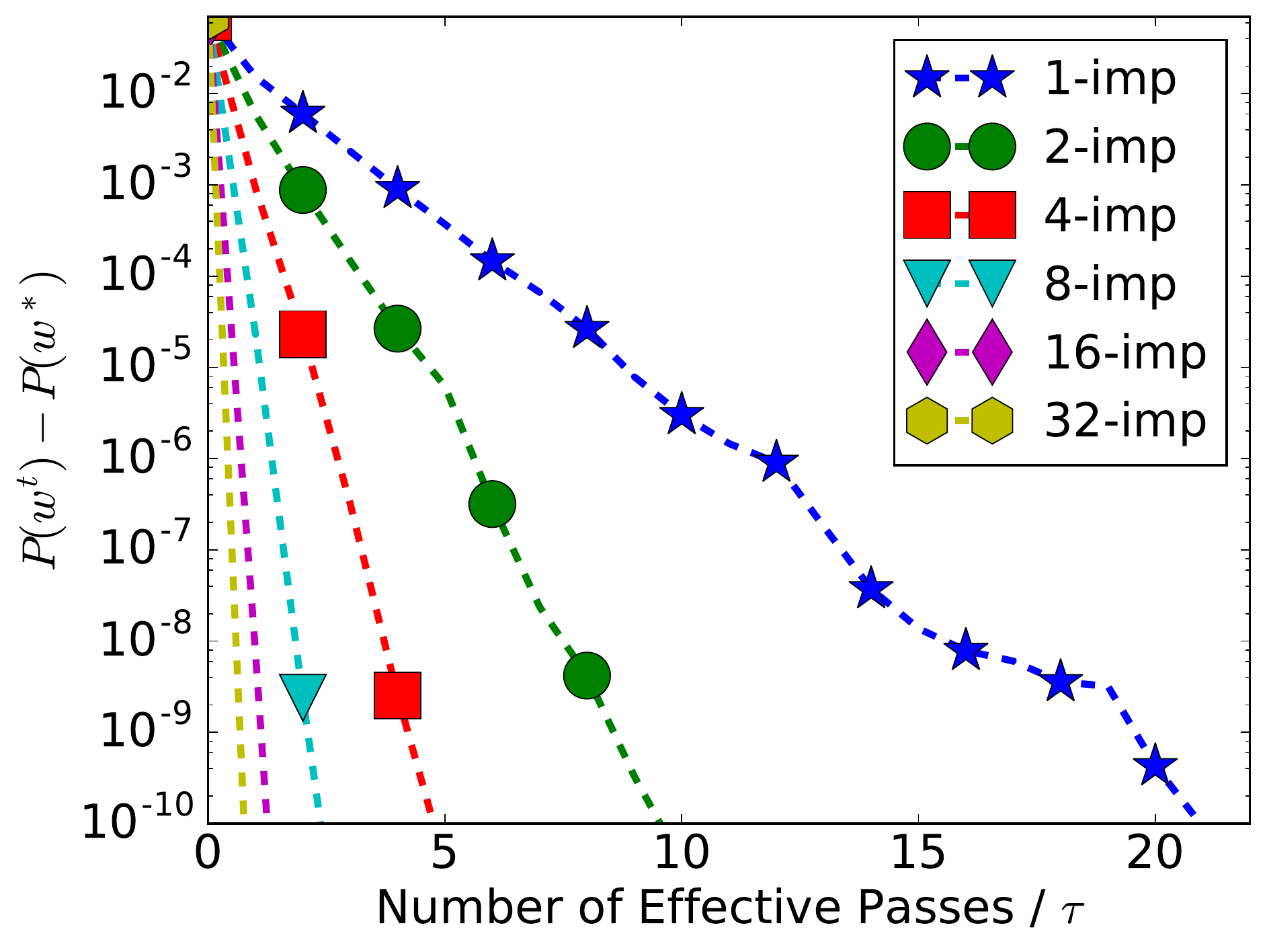}}
\hfill
\subfigure[chisq100, $\tau$-nice (left), $\tau$-importance (right)]{\includegraphics[width=0.5\columnwidth]{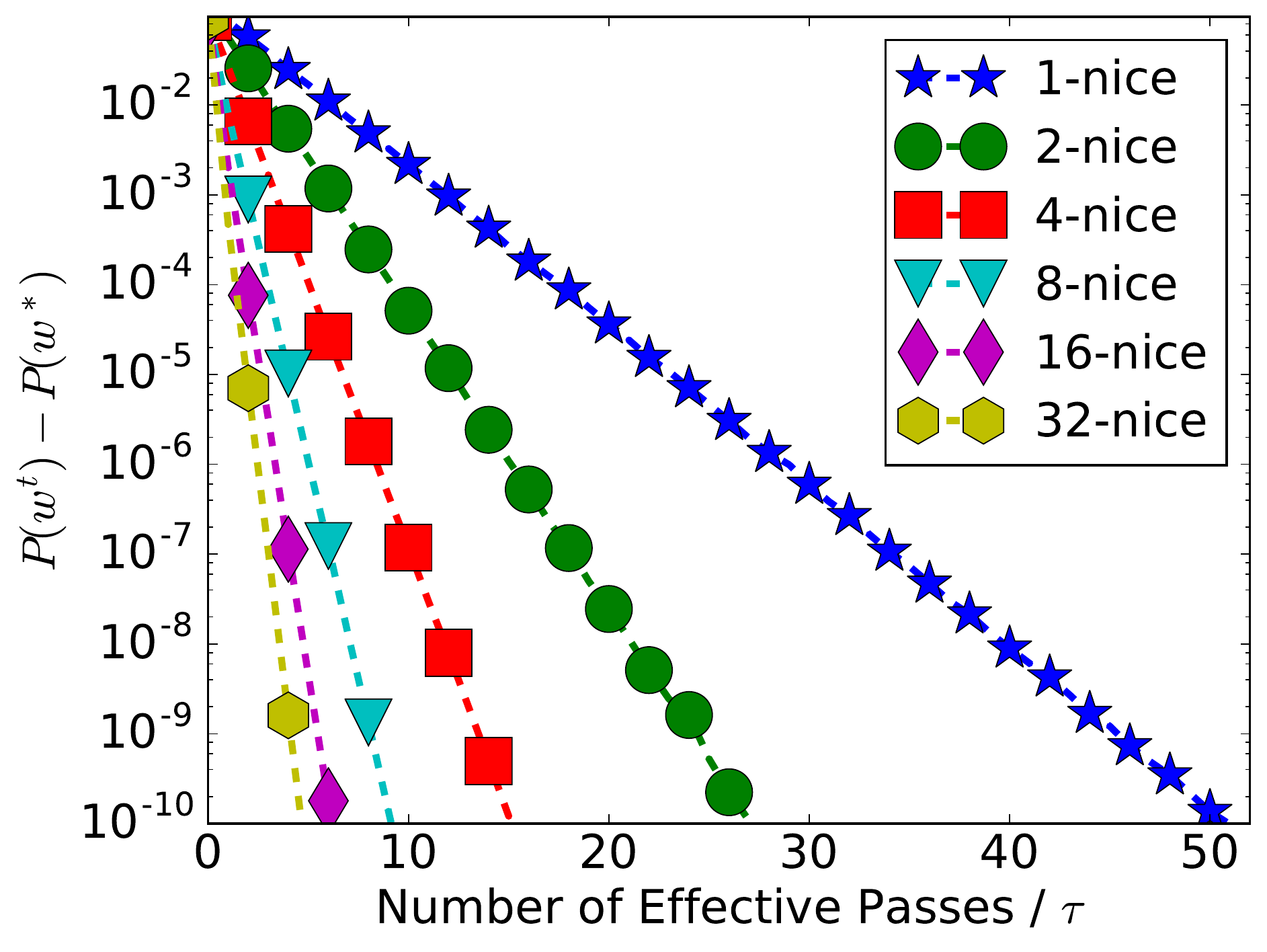}\includegraphics[width=0.5\columnwidth]{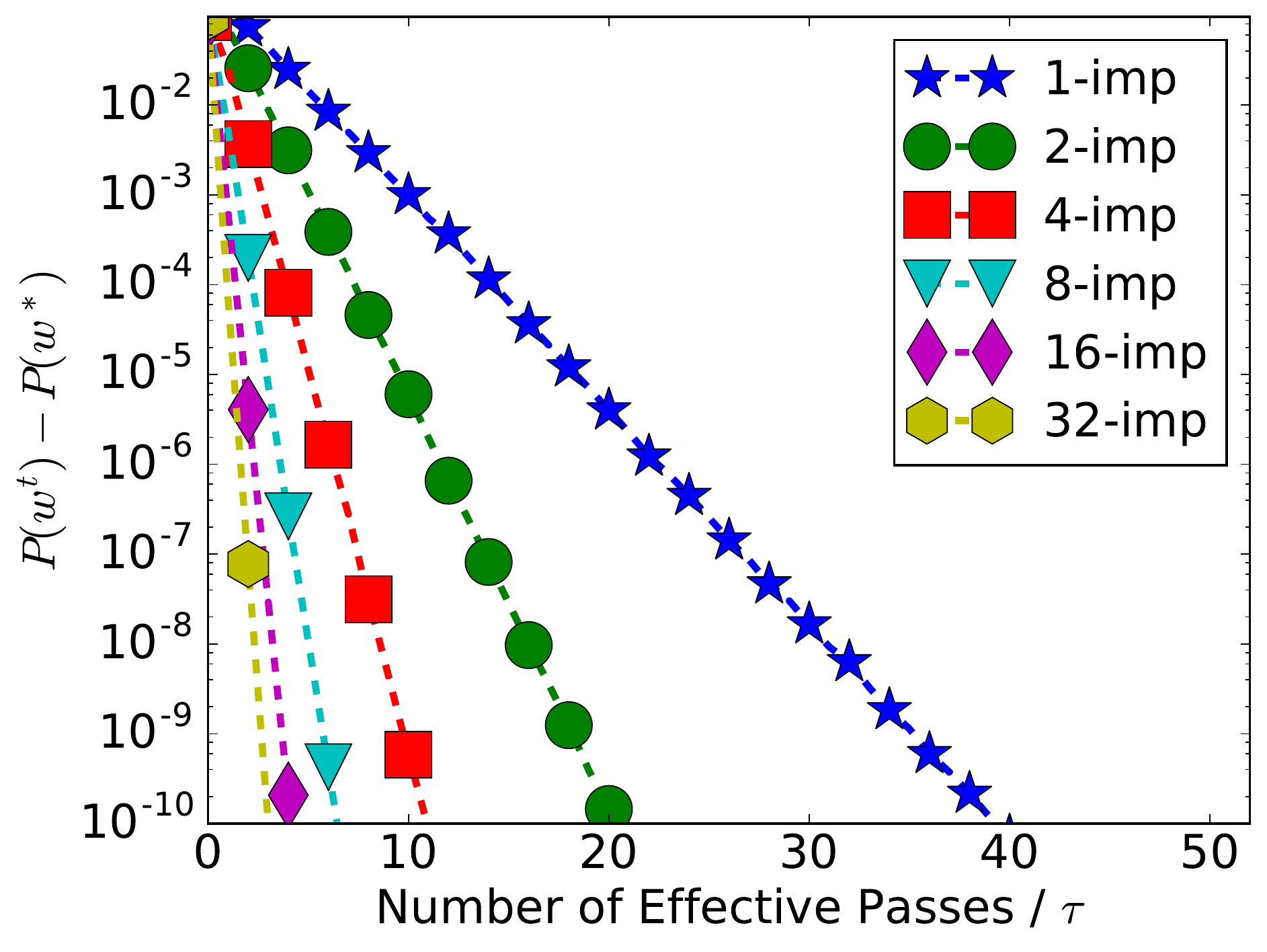}}
\hfill
\subfigure[chisq10, $\tau$-nice (left), $\tau$-importance (right)]{\includegraphics[width=0.5\columnwidth]{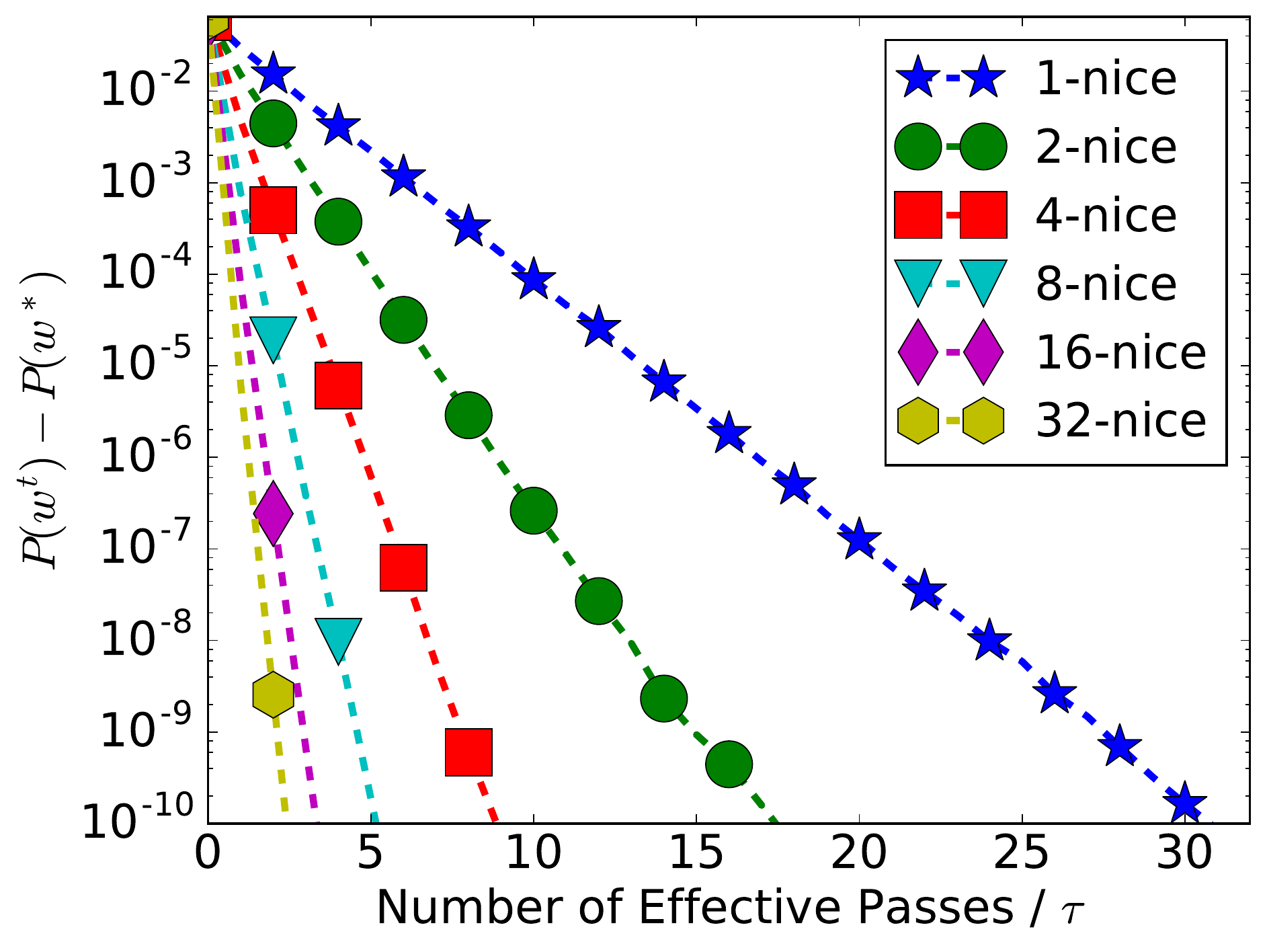}\includegraphics[width=0.5\columnwidth]{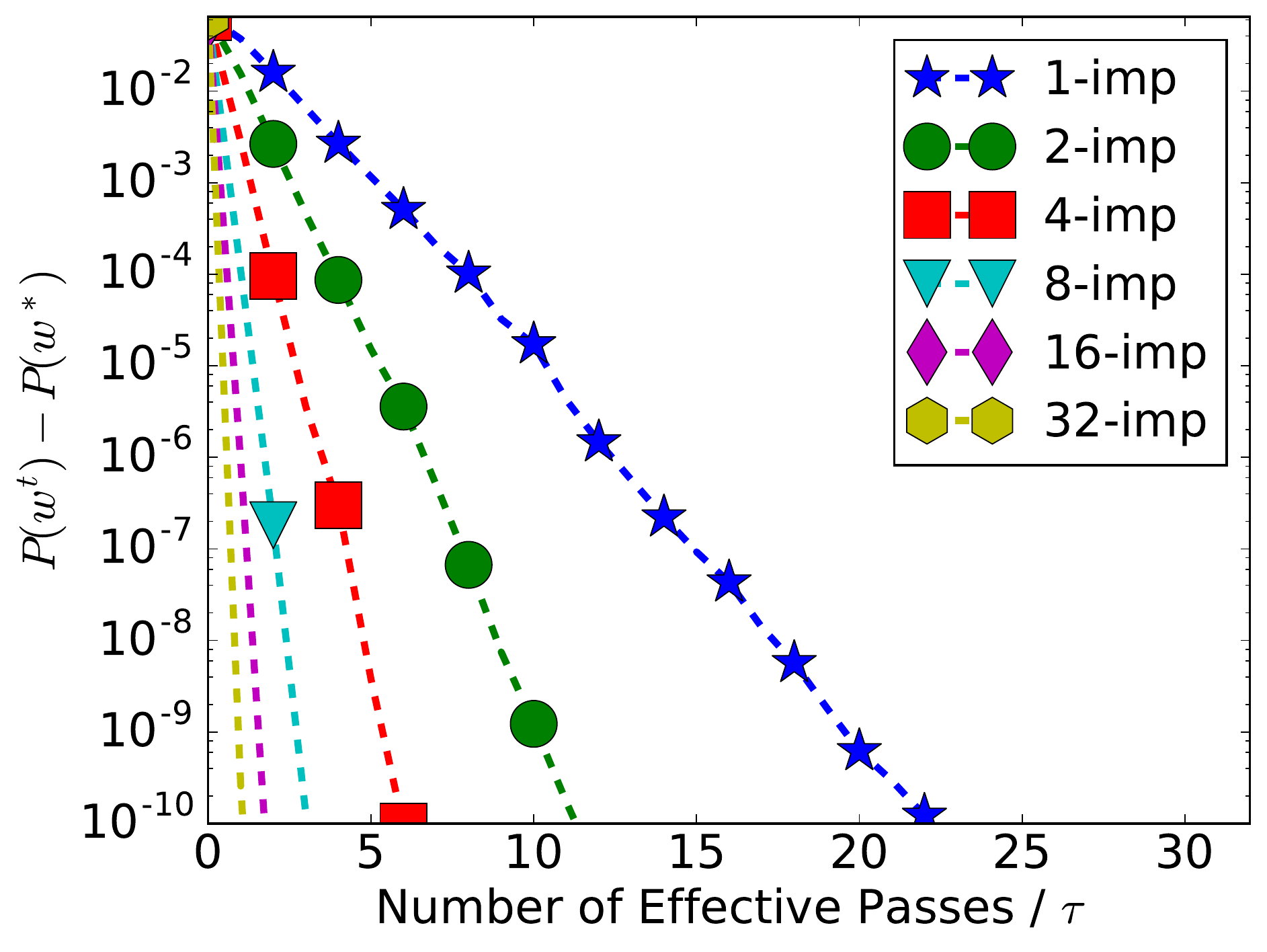}}
\hfill
\subfigure[chisq1, $\tau$-nice (left), $\tau$-importance (right)]{\includegraphics[width=0.5\columnwidth]{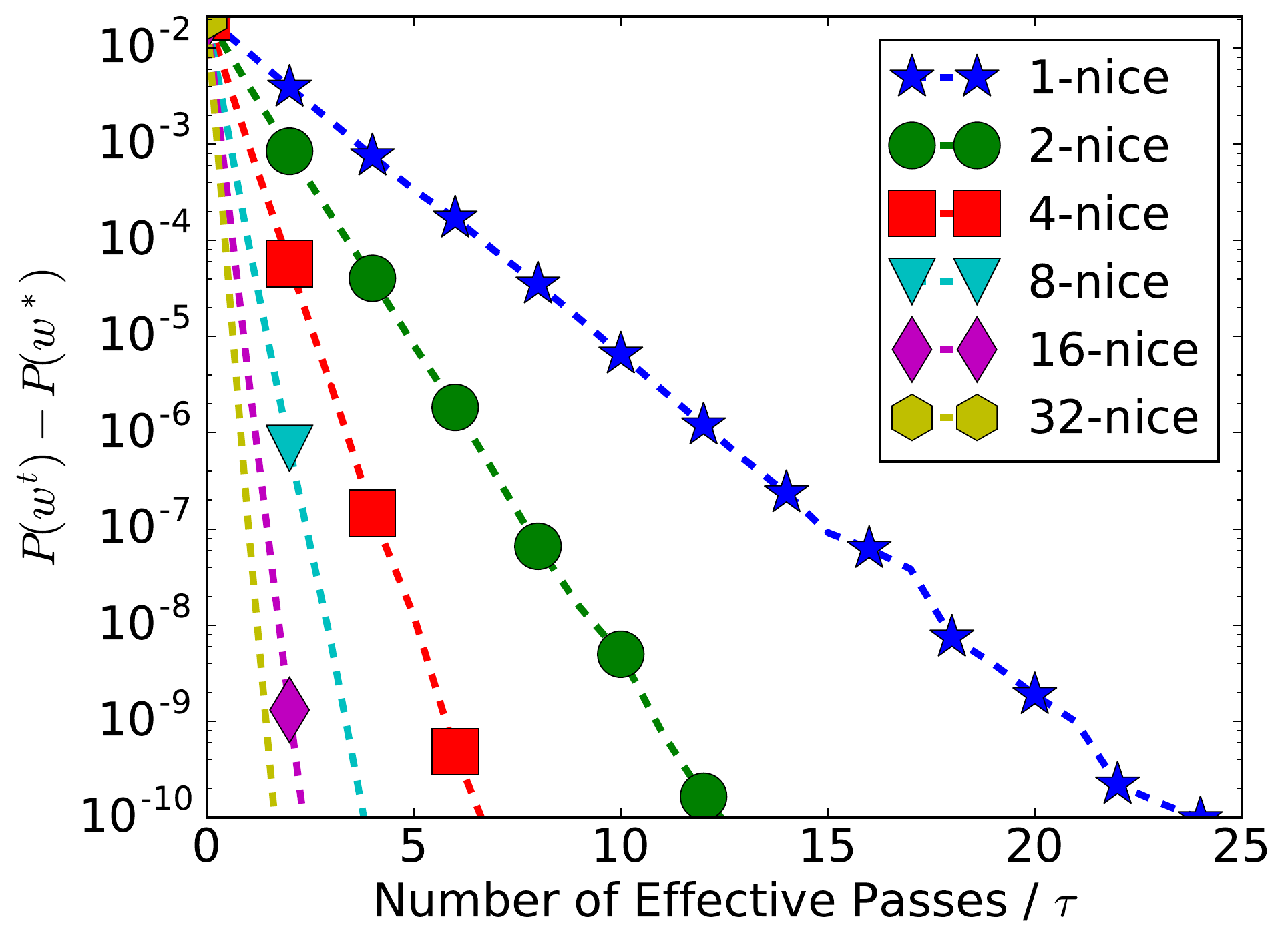}\includegraphics[width=0.5\columnwidth]{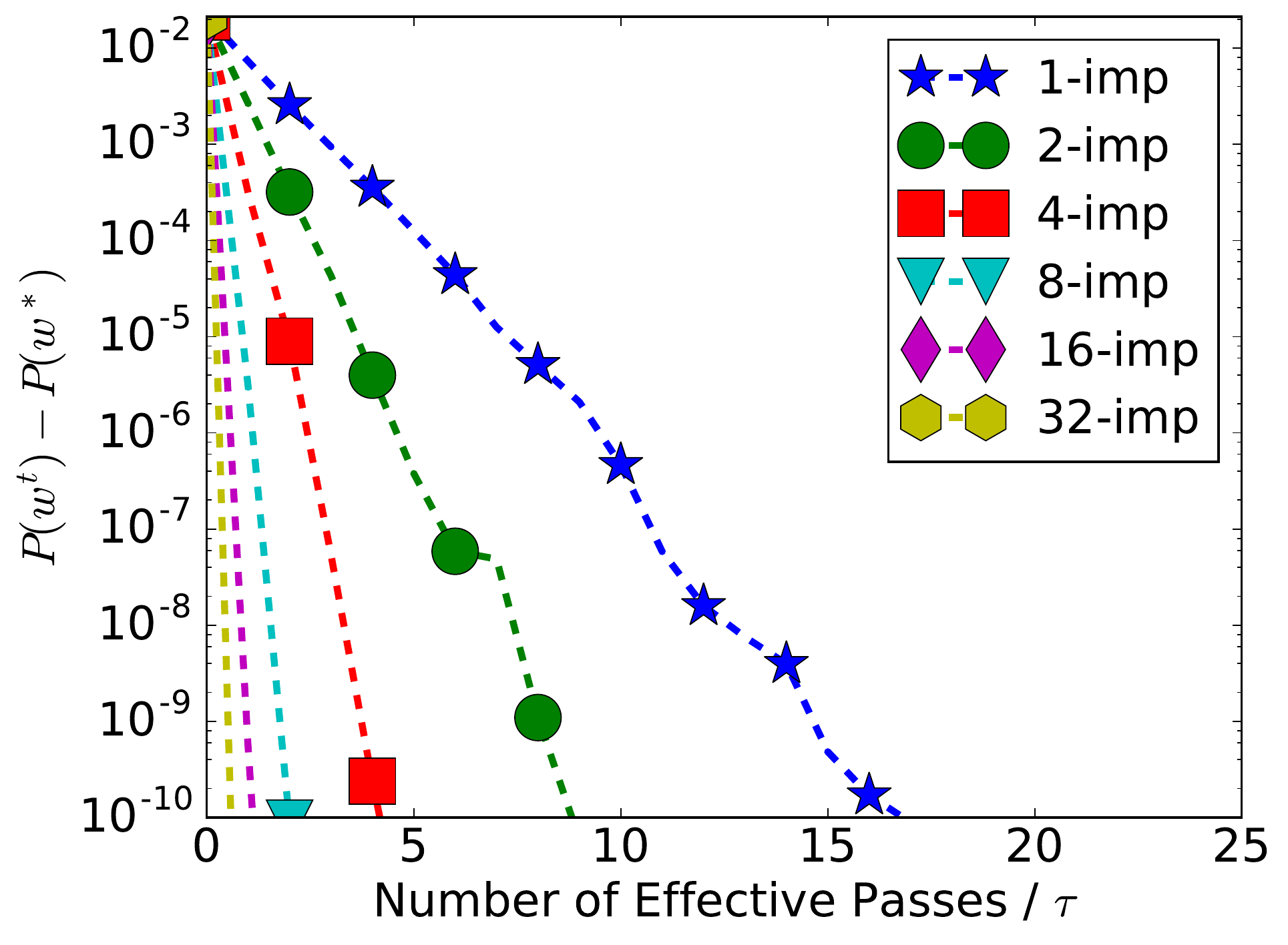}}
\hfill
\subfigure[extreme, $\tau$-nice (left), $\tau$-importance (right)]{\includegraphics[width=0.5\columnwidth]{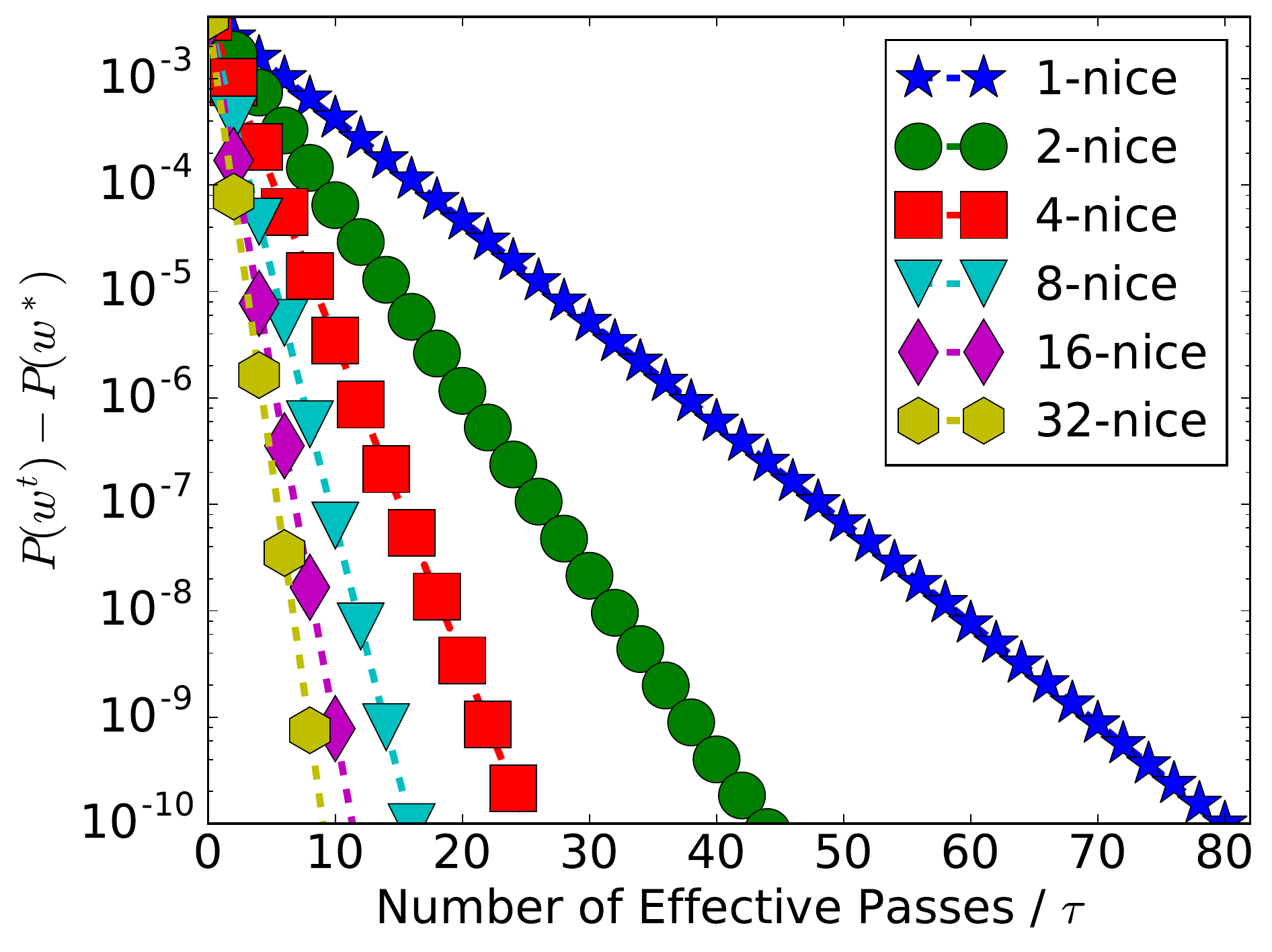}\includegraphics[width=0.5\columnwidth]{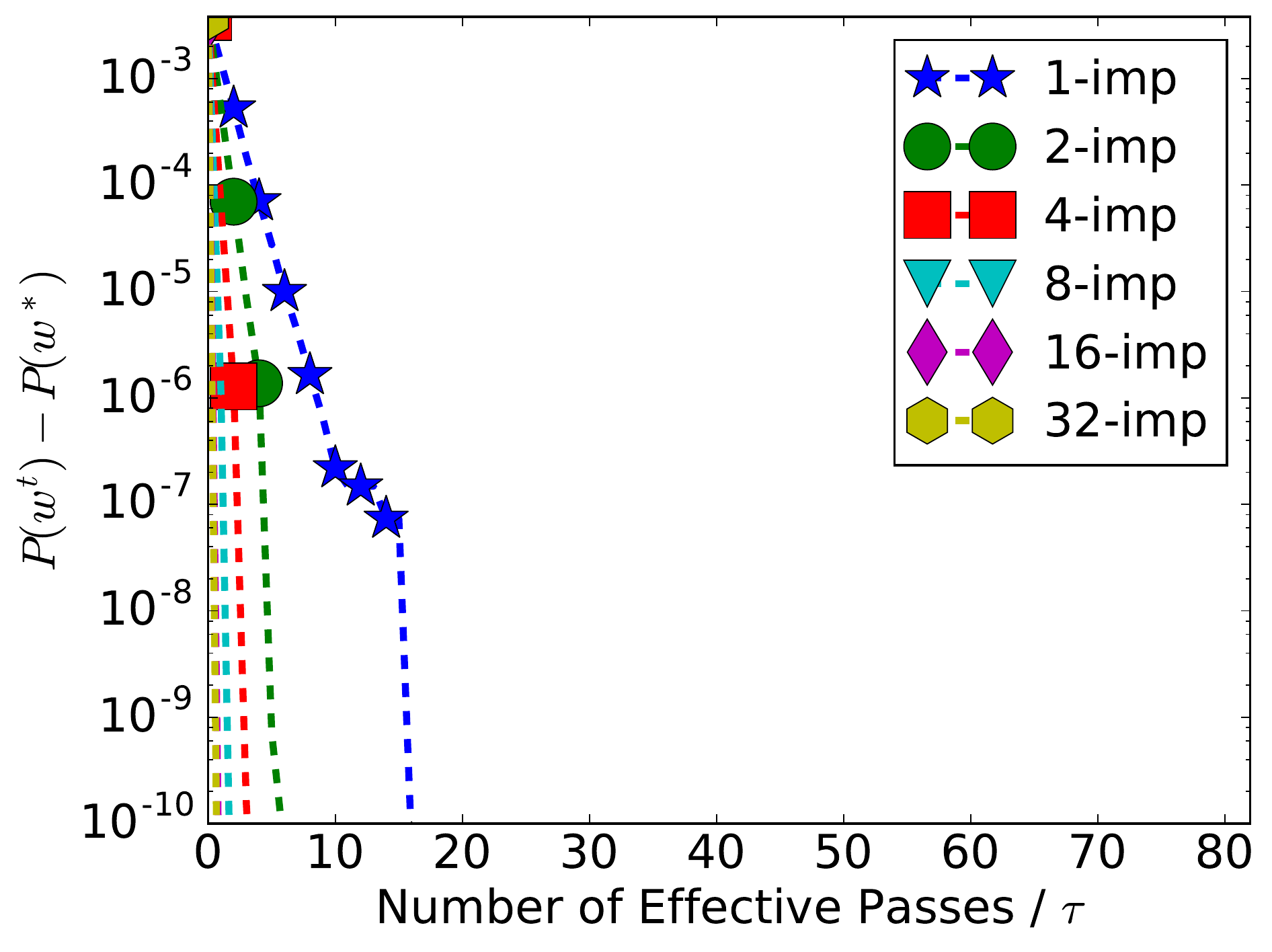}}
\caption{Artificial datasets from Table \ref{tab:dists} with $\omega = 0.1$}
\label{fig:artificial_0.1}
\end{minipage}
\end{figure}

\clearpage
\section{Proof of Theorem~\ref{thm:ESO}}

\subsection{Three lemmas}

We first establish three lemmas, and then proceed with the proof of the main theorem. With each sampling $\hat{S}$ we associate an $n\times n$ ``probability matrix''  defined as follows: $\bP_{ij}(\hat{S}) = \Prob(i\in \hat{S}, j\in \hat{S})$. Our first lemma  characterizes  the probability matrix of the bucket sampling.

\begin{lemma} \label{lem:09hssss} If $\hat{S}$ is a bucket sampling, then
\begin{equation}\label{eq:Px}
\mathbf{P}(\hat{S}) = p p^\top \circ (\mathbf{E} - \mathbf{B}) + \Diag(p),
\end{equation}
where  $\mathbf{E}\in \R^{n\times n}$ is the matrix  of all ones,  
\begin{equation}\label{eq:B} \mathbf{B} \eqdef \sum_{l=1}^\tau \bP(B_l),\end{equation}
and $\circ$ denotes the Hadamard (elementwise) product of matrices. Note that $\bB$ is the 0-1 matrix given by $\mathbf{B}_{ij} = 1$ if and only if $i,j$ belong to the same bucket $B_l$ for some $l$.
\end{lemma}
\begin{proof}
Let $\bP = \bP(\hat{S})$. By definition  \[\bP_{ij} =  \begin{cases} ~p_i \quad& i=j \\  ~p_i p_j \quad& i\in B_l,~j\in B_k,~l\neq k \\  ~0 \quad&\text{otherwise}. \end{cases}\]
It only remains to compare this to \eqref{eq:Px}.
\end{proof}

\begin{lemma} Let $J$ be a nonempty subset of $[n]$, let $\bB$ be as in Lemma~\ref{lem:09hssss} and put $\omega_J' \eqdef |\{l : J\cap B_l \neq \emptyset\}|$. Then
\begin{equation}\label{eq:suhisihis8s}  \bP(J) \circ \bB \succeq \frac{1}{\omega'_J} \bP(J).\end{equation}
\end{lemma}
\begin{proof}
For any $h\in \R^n$, we have
\[ h^\top \bP(J) h =  \left(\sum_{i\in J} h_i\right)^2 = \left(\sum_{l=1}^\tau \sum_{i\in J\cap B_l}  h_i\right)^2 
 \leq \omega_J' \sum_{l=1}^\tau \left( \sum_{i\in J\cap B_l }  h_i\right)^2   =\omega_J' \sum_{l=1}^\tau h^ \top \bP(J\cap B_l)h, \]
\label{eq:iuhdgd899}
where we used the Cauchy-Schwarz inequality. Using this, we obtain
\[ \bP(J) \circ \bB  \overset{\eqref{eq:B}}{=} \bP(J)\circ  \sum_{l=1}^ \tau \bP(B_l)=  \sum_{l=1}^ \tau \bP(J)\circ\bP(B_l)  =\sum_{l=1}^\tau \bP(J\cap B_l) 
\overset{\eqref{eq:iuhdgd899}}{\succeq} \frac{1}{\omega'}\bP(J). \]
\end{proof}

\begin{lemma} Let $J$ be any nonempty subset of $[n]$ and $\hat{S}$ be a bucket sampling. Then
\begin{equation}\label{eq:9s989s8sss} \bP(J ) \circ pp^\top\preceq \left(\sum_{i\in J} p_i \right) \Diag(\bP(J\cap \hat{S})) . \end{equation}
\end{lemma}
\begin{proof}
Choose any $h\in \R^n$ and note that \[ h^\top (\bP(J ) \circ pp^\top) h  = \left(\sum_{i\in J} p_i h_i\right)^2 =  \left(\sum_{i\in J} x_i y_i\right)^2, \] where $x_i  = \sqrt{p_i} h_i$ and $y_i = \sqrt{p_i}$. It remains to apply the Cauchy-Schwarz inequality: \[\sum_{i\in J} x_i y_i \leq \sum_{i\in J} x_i^2 \sum_{i\in J} y_i^2\] and notice that the $i$-th element on the diagonal of $\bP(J\cap \hat{S})$ is $p_i$ for $i\in J$ and 0 for $i\notin J$
\end{proof}

\subsection{Proof of Theorem~\ref{thm:ESO}}

By Theorem 5.2 in \cite{ESO}, we know that inequality \eqref{eq:ESO} holds for parameters $\{v_i\}$ set to \[ v_i = \sum_{j=1}^d \lambda'(\bP(J_j \cap \hat{S}))\mathbf{X}_{ji}^2 ,\] where $\lambda'(\mathbf{M})$ is the largest normalized eigenvalue of symmetric matrix $\mathbf{M}$ defined as
\[\lambda'(\mathbf{M}) \eqdef \max_{h} \left\{h^\top \bM h \;:\; h^\top \Diag(\bM) h \leq 1\right\}.\]
Furthermore, 
\begin{align*}
\bP(J_j\cap \hat{S}) &= \bP(J_j) \circ \bP(\hat{S}) \\ 
&\overset{\eqref{eq:Px}}{=}  \bP(J_j)  \circ p p^\top - \bP(J_j)   \circ p p^\top  \circ \bB + \bP(J_j)   \circ \Diag(p) \\
&\overset{\eqref{eq:suhisihis8s}}{\preceq}   \left(1 -  \frac{1}{\omega'_J} \right) \bP(J_j)  \circ p p^\top + \bP(J_j)   \circ \Diag(p) \\ &\overset{\eqref{eq:9s989s8sss}}{\preceq} \left(1 - \frac{1}{\omega'_J} \right) \delta_{j} \Diag(\bP(J_j \cap \hat{S}))   + \Diag(\bP(J_j \cap \hat{S})), 
\end{align*}
whence $\lambda'(\bP(J_j\cap \hat{S}) ) \leq 1 + \left(1 - 1/\omega'_J \right) \delta_{j}$, which concludes the proof.

\clearpage

\bibliography{minibatching}
\bibliographystyle{plainnat}

\end{document}